\documentclass[twoside]{article}

\usepackage[accepted]{aistats2020}
\usepackage{amsmath}
\usepackage{amssymb}
\usepackage{graphicx}
\usepackage{dsfont}
\usepackage{algorithm}
\usepackage{color}
\usepackage{setspace}
\usepackage{appendix}
\usepackage{comment}
\usepackage{multirow}
\usepackage{caption}
\usepackage{subcaption}
\usepackage{natbib}

\usepackage{url}
\usepackage{xr}

\usepackage{amsthm}
\newtheorem{theorem}{Theorem}
\newtheorem{lemma}{Lemma}

\newtheorem{proposition}{Proposition}

\def\ham{h_{a,M}}
\def\Bham{B(\ham)}

\DeclareMathOperator*{\argmax}{argmax}
\DeclareMathOperator*{\argmin}{argmin}

\allowdisplaybreaks

\begin{document}

%

%

\twocolumn[

\aistatstitle{On Thompson Sampling for Smoother-than-Lipschitz Bandits}

\aistatsauthor{ James A. Grant \And David S. Leslie }

\aistatsaddress{ STOR-i Centre for Doctoral Training, \\ Lancaster University. \And  Department of Mathematics and Statistics, \\ Lancaster University 
and PROWLER.io. } 
]

\begin{abstract}
  Thompson Sampling is a well established approach to bandit and reinforcement learning problems. However its use in continuum armed bandit problems has received relatively little attention. We  provide the first bounds on the regret of Thompson Sampling for continuum armed bandits under weak conditions on the function class containing the true function and sub-exponential observation noise. Our bounds are realised by analysis of the eluder dimension, a recently proposed measure of the complexity of a function class, which has been demonstrated to be useful in bounding the Bayesian regret of Thompson Sampling for simpler bandit problems under sub-Gaussian observation noise. We derive a new bound on the eluder dimension for classes of functions with Lipschitz derivatives, and generalise previous analyses in multiple regards. 
\end{abstract}

\section{Introduction}

Thompson Sampling (TS) \citep{Thompson1933,Russo2018} is a Bayesian approach to sequential decision making problems that has been widely applied and found to have both strong empirical performance and desirable theoretical properties. A major advantage of TS is it can typically be extended to new problems in a straightforward manner, with empirical success and without a need to tune parameters or rely on detailed theory to design an algorithmic structure. Two of its shortcomings, however, are that it may be more challenging to analyse theoretically than related approaches, and that for complex problems it may often only be implemented approximately, since it relies on draws from the distribution on the reward function. As a result of these challenges, theoretical guarantees on TS are mostly limited to parametric bandit problems.

\citet{RussoVanRoy2014} introduced a general analytical technique, based on a measure of problem complexity called the eluder dimension, and applied it to analyse the performance of TS on a family of parametric bandit problems. In this paper we show how this eluder-dimension-based analysis can be generalised substantially. We provide new order-optimal performance guarantees for TS on non-parametric continuum-armed bandit problems whose reward functions have a number of Lipschitz derivatives. These guarantees provide insights into the performance of exact TS which significantly advance current understanding, and also serve as empirical benchmarks and analytical tools for future analyses of approximate TS.

\subsection{Bandit Problems}

Multi-armed bandit (MAB) problems \citep{Lattimore2018} are classic models of exploration-exploitation dilemmas in sequential decision making problems. Among the most general of these is the stochastic Continuum-Armed Bandit (CAB) problem \citep{Agrawal1995}. The CAB models a scenario in which a decision-maker repeatedly selects \emph{actions}, represented by elements $a$ of an \emph{action set} $\mathcal{A} \subseteq \mathbb{R}^d$. Taking an action grants the decision-maker a \emph{reward} which is a noisy perturbation of some function $f: \mathcal{A} \rightarrow \mathbb{R}$, called the \emph{reward function}, at the selected action $a$. The decision-maker's objective is to maximise the sum of the rewards they receive over some finite number of actions, without knowledge of $f$.

Effective strategies toward realising this objective will exhibit an appropriate balance between selecting `exploratory' actions, which aim to learn the function $f$ across $\mathcal{A}$ to gain confidence in the location of its maximum, and `exploitative' actions, which target regions where $f$ is empirically suggested to take large values in order to maximise the sum of rewards. This need to balance between exploration and exploitation is present in simpler bandit problems (e.g. those  where the set $\mathcal{A}$ is finite, or where the function $f$ is known to have a simple parametric form). However in the more general CAB setting, where we have limited assumptions on $f$, realising this balance has historically been more challenging.

\subsection{Thompson Sampling}

Thompson Sampling (TS), also referred to as \emph{posterior sampling}, is a Bayesian approach to sequential decision making problems which aims to achieve an appropriate balance between exploration and exploitation through randomisation \citep{Thompson1933,Russo2018}. 

Over a sequence of rounds $t \in \mathbb{N}$, the decision-maker utilising TS selects actions by sampling a function $\tilde{f}_t$ from their current posterior belief on the form of the true reward function $f$, and then selecting an action $a_t \in \mathcal{A}$ which maximises $\tilde{f}_t$ - i.e. an action that would be expected to contribute optimally to the cumulative reward if $\tilde{f}_t$ were the true reward function. Figure \ref{fig::TS_example} illustrates a single step of TS on a CAB.

\begin{figure*}[htbp]
  \includegraphics[width=\textwidth,height=3.5cm]{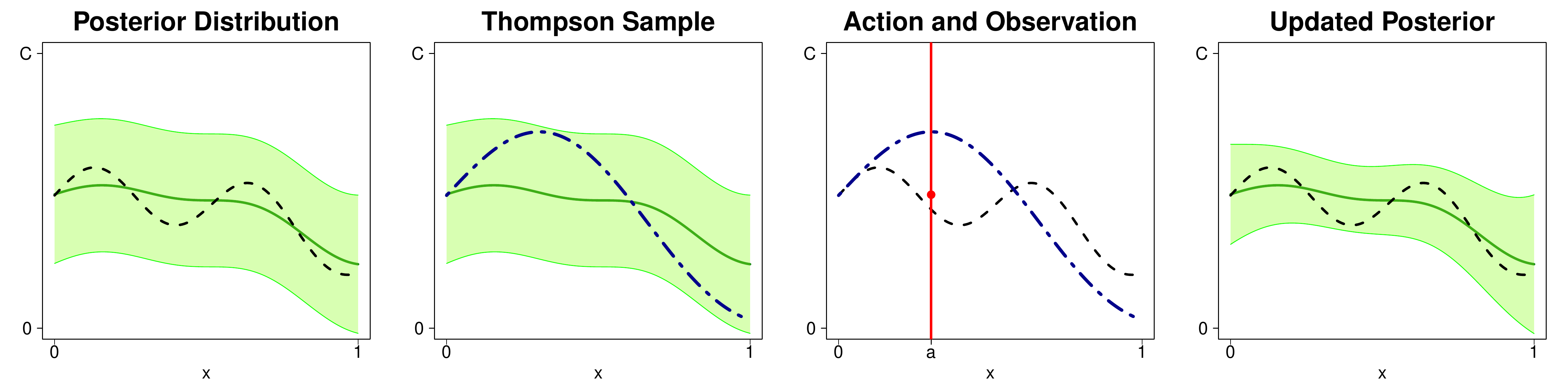}
  \caption{Illustrative example of Thompson Sampling in a round $t \in \mathbb{N}$. The first (leftmost) pane displays a credible interval of a posterior $p_{t-1}$ on $\mathcal{F}$ in green, and a true reward function $f_0 \in \mathcal{F}$ in black. In the second pane the blue curve represents a function $\tilde{f}_t$ sampled from $p_{t-1}$. In the third pane, the choice of action $a_t \in \argmax_{a \in [0,1]}\tilde{f}_t(a)$ is highlighted in red, along with a reward observation $R(a_t)$ highlighted as a red dot. Finally, the fourth pane displays the posterior $\pi_t$ updated on the basis of $(a_t,R(a_t))$.} \label{fig::TS_example}
\end{figure*}

TS therefore encourages exploration since the posterior distribution has more uncertainty on the value of $f$ in regions of $\mathcal{A}$ where few actions have been selected. TS gradually favours exploitation as the posterior distribution naturally contracts around $f$ and sampled functions have maxima in similar locations to $f$. 

TS has been shown empirically to be highly effective in a wide range of bandit problems \citep{ChapelleLi2011,Russo2018}, and theoretical results \cite[etc.]{May2012,Kaufmann2012thompson,Agrawal2012,RussoVanRoy2014} have confirmed this in numerous settings where the reward function may be written in terms of a finite set of parameters. The idea of TS extends readily to nonparametric reward functions, but has received little attention in the literature. We believe that this is, in part, due to the challenges of theoretical analysis and precise inference in complex Bayesian models.

Recently, tools have been developed that mean these challenges are not as insurmountable as they once were. Algorithms for approximate Bayesian inference, such as sequential Monte Carlo and variational inference, have become increasingly sophisticated in recent years, to the point that high quality approximations to TS are now feasible \citep{LuVanRoy2017,UrteagaWiggins2018a,UrteagaWiggins2018b}. 

On the theoretical side, \cite{RussoVanRoy2014} introduce a general analytical approach for deriving performance guarantees for TS in bandit problems. This method is based on characterising the entropy of the function class in which possible reward functions are contained, via a quantity called the \emph{eluder dimension}. In \cite{RussoVanRoy2014} this technique was successfully used to analyse the performance of TS on bandit problems with (generalised) linear reward functions.

\citeauthor{RussoVanRoy2014}'s technique can be applied much more widely.  In this paper, we show that the method for deriving performance guarantees in terms of the eluder dimension can be extended to CABs whose reward functions are members of non-parametric function classes. We show that TS achieves order optimal performance subject to sufficient conditions on the smoothness of these functions (that they have infinitely many Lipschitz derivatives). We further formalise the framework in which this is achievable in the following subsection.

\subsection{Model} \label{sec::lipTSmodel}
We specify a general CAB problem as a tuple $(\mathcal{A},f_0,p_\eta)$, where $\mathcal{A}$ is the set of available actions, $f_0:\mathcal{A} \rightarrow \mathbb{R}$ is the unknown reward function, and $p_\eta$ is the distribution of the reward noise. We model $f_0$ as being a sample from $p_0$, a non-parametric prior on a function class $\mathcal{F}$ whose nature we will specify later.

In a sequence of rounds $t \in [T] \subseteq \mathbb{N}$, the decision-maker selects an action $a_t \in\mathcal{A}$ and receives a reward $R_t=f_0(a_t)+\eta_t$, which is a noisy perturbation of the reward function at $a_t$ with noise terms $\eta_t$ distributed according to $p_\eta$. Let $\mathcal{H}_{t}=\sigma(a_1,R_1,\dots,a_{t},R_{t})$ be the $\sigma$-algebra induced by the history of the first $t$ actions and rewards. We assume that for $t \in [T]$, $\eta_t$ is $(\sigma^2,b)$-sub-exponential conditioned on $a_t$, meaning \begin{equation}
\mathbb{E}\big(e^{\lambda\eta_t}|\mathcal{H}_{t-1},a_t \big) \leq e^{\frac{\lambda^2 \sigma^2}{2}}, \quad \forall \enspace |\lambda| \leq \frac{1}{b}. \label{eq::subexpprop}
\end{equation} The noise terms $\eta_t$ are also assumed to be conditionally independent given the actions $a_t$, $t \in [T]$.

We are interested in the performance of TS as a policy to select actions $a_t$ for $t \in [T]$. Let $p_{t}$ denote the posterior distribution on $f_0$ conditioned on $\mathcal{H}_t$ and let $\tilde{f}_t$ be a sample from $p_{t}$. The TS approach is the one which chooses an action $a_t \in \argmax_{a \in \mathcal{A}} \tilde{f}_{t-1}(a)$ in round $t$, breaking ties arbitrarily if the maximiser is non-unique.

We principally concern ourselves with the Bayesian regret of TS in $T$ rounds, given as 
\begin{equation}
BR(T)= \mathbb{E}_{p_0}\bigg(\sum_{t=1}^T \max_{a \in \mathcal{A}} f_0(a)-f_0(a_t)\bigg), \label{eq::BRdef}
\end{equation} where $\mathbb{E}_{p_0}$ denotes expectation with respect to the prior $p_0$. In particular, we are interested in bounding the Bayesian regret as a function of $T$ for particular $\mathcal{A}$ and $\mathcal{F}$, and the order with respect to $T$ that such bounds possess. We will also derive lower bounds on the frequentist regret of any algorithm. The frequentist regret, \begin{displaymath}
Reg(T) = T \max_{a \in \mathcal{A}}f_0(a) -\sum_{t=1}^T \mathbb{E}\big(f(a_t) \, | \, f=f_0\big),
\end{displaymath} is similar to the Bayesian regret \eqref{eq::BRdef}, the only difference being that the expectation is (as the name suggests) a frequentist expectation conditioned on a fixed reward function $f_0$, whereas the Bayesian regret additionally takes the expectation of the frequentist regret with respect to this reward function $f_0$. Frequentist regret bounds which are available for any function $f_0$ may then be seen as uniform bounds on the Bayesian regret for any prior. We will assess the performance of TS by considering the gap (if any) between the order of the upper and lower bounds. We note that while the analytical tools to upper bound the frequentist regret of non-parametric TS are not currently available, the Bayesian regret is still a useful, and indeed natural, metric to consider in the Bayesian framework.

\subsection{Related Work}
Numerous authors have studied the frequentist regret of TS in bandit problems, with varying assumptions on the action set, feedback mechanism and reward noise distribution \citep{May2012,Agrawal2012,Kaufmann2012thompson,Korda2013,Komiyama2015,Wang2018}. None of these works address the fully nonparametric CAB. 

Study of the Bayesian regret of TS originated with \cite{RussoVanRoy2014}. Using the eluder dimension measure of the complexity of the reward function class they derived a bound on the Bayesian regret of TS for general action sets and parametric reward function classes. They specialise this to bandit problems with (generalised) linear reward functions. Quadratic functions and applications in model-based reinforcement learning are considered by \cite{OsbandVanRoy2014}. Our paper considers a more substantial extension of this technique to reward functions with Lipschitz derivatives.

As already noted, one challenge in deploying Thompson sampling is that sampling from the requisite posterior distributions can only be carried out approximately, thus rendering the theoretical results obsolete. Recently, \cite{PhanEtAl2019} have studied the regret of approximations of TS, demonstrating a link between the (assumed to be fixed) error of the approximation and the regret of TS for $K$-armed bandit problems. We do not address this aspect of the theory in the current article.

The main alternatives to TS in CAB problems, are upper confidence bound (UCB) approaches. These methods, which follow from ideas in \cite{LaiRobbins1985} and \cite{AuerEtAl2002} for simpler $K$-armed bandits, encourage exploration by making decisions with respect to optimistic estimates of the reward function. Certain UCB methods have been shown to have order-optimal regret bounds in certain CAB problems. These approaches typically employ an \emph{adaptive discretisation} structure, where the action space available at time $t$ is limited to some $\mathcal{A}_t \subset \mathcal{A}$ to force an appropriate level of exploration. 

In particular, the `zooming algorithm' of \cite{Kleinberg2008} maintains a finite set of `active arms' in $\mathcal{A}$ and only selects actions from within this set. The size of this set is gradually increased by adding arms with high exploitative or exploratory value. The frequentist regret of the zooming algorithm can be shown to be bounded as $O(T^{2/3})$ for the CAB with Lipschtiz reward function and sub-Gaussian noise. \cite{LuEtAl2019} extend these results to heavy-tailed reward noise distributions. This rate is known to be optimal, as \cite{Kleinberg2005} demonstrate that the best achievable regret is $\Omega(T^{2/3})$ across all possible problem instances. 

A similar approach is the Hierarchical Online Optimisation (HOO) algorithm of \cite{Bubeck2011}, which discretises the action space according to a tree-based algorithm. In \cite{Bubeck2011} a yet more general 
bandit problem is studied where the action set may be any appropriate metric space. HOO is shown to have frequentist regret bounded with order $O(T^{(d'+1)/(d'+2)})$ where $d'>0$ is a parameter related to the covering number of the metric space, and nature of the possible reward functions. Recent works of \cite{Slivkins2019} and \cite{KleinbergEtAl2019} provide more extensive summaries of bandits on metric spaces.

Apart from this, the special case of a CAB problem with sub-Gaussian noise whose reward function is a sample from a Gaussian process (GP), sometimes referred to as \emph{GP optimisation}, has received particular attention. This setting is more restrictive than ours, but is popular because of its intersection with common modelling assumptions in Bayesian optimisation \citep{ShahriariEtAl2016}. The GP-UCB approach of \cite{SrinivasEtAl2009,SrinivasEtAl2012} exploits the closed-form of the GP posterior to calculate an upper confidence function (a combination of the mean and variance of the posterior GP) at each round which is optimised to select actions and enjoys optimal order regret. In this setting both GP-UCB and a GP-based variant of TS can be shown to have $O(\sqrt{T\log(T)})$ Bayesian regret \citep{SrinivasEtAl2012,RussoVanRoy2014}, which is optimal for the problem up to a logarithmic factor. 



\subsection{Key Contributions and Structure}
Our main contribution is a bound on the Bayesian regret of Thompson Sampling applied to Continuum-armed Bandits where the reward function is a sample from a prior distribution on the class of bounded functions functions with $M \in \mathbb{N}$ Lipschitz smooth derivatives and the reward noise is sub-exponentially distributed. As far as we are aware this is the first analysis of the performance of TS based on nonparametric inference that considers such a general framework. We derive a {$O(T^{(2M^2+11M+10)/(4M^2+14M+12)})$} 
Bayesian regret bound, 
which approaches $O(\sqrt{T})$ as $M\to\infty$. 

In the process of proving this result we give the first bound on the $\epsilon$-eluder dimension of Lipschitz function classes, and we extend bounds on the Bayesian regret of Thompson Sampling for bandit problems with (generalised) linear reward function to the sub-exponential reward noise setting.

{Furthermore we derive an $\Omega(T^{(M+2)/(2M+3)})$ lower bound on regret. There is thus an $O(T^{(3M+2)/(4M^2+14M+12)})$ gap between the lower and upper bounds, which is small for large $M$. It is an open question as to whether this gap is due to TS being suboptimal, or whether the upper (or lower) bounds we have derived are not tight.}

The remainder of the material is organised as follows. In Section \ref{sec::genBRbounds} we present an extension of \citet{RussoVanRoy2014}'s general bound on the Bayesian regret. We specialise this to problems where the reward function class has Lipschitz derivatives in Section \ref{sec::SpecificFunClasses}, and conclude with a discussion in Section 4. Proofs are relegated to the Appendices.

\section{General Bound on the Bayesian Regret} \label{sec::genBRbounds}
We first give a bound on the Bayesian regret for general function classes, $\mathcal{F}$, and action sets, $\mathcal{A}$ - including the CAB whose reward function has Lipschitz derivatives. Our result is similar to, but more general than, Proposition 10 of \cite{RussoVanRoy2014}. Their result holds only under sub-Gaussian noise on the reward observations, and has less flexibility in terms of being able to tune the terms based on the properties of $\mathcal{F}$. Our result has such added flexibility and applies to sub-exponential rewards.

Both our bound and that of \cite{RussoVanRoy2014} are expressed in terms of measures of the complexity of the function class $\mathcal{F}$. This is natural, since in more complex function classes, it will be more challenging to learn the true function.
Specifically, two notions of the complexity of $\mathcal{F}$ are of interest, the $\epsilon$\emph{-eluder dimension}, and \emph{ball-width function}, which we introduce below. 

Firstly, to define the $\epsilon$-eluder dimension, we introduce the notion of $\epsilon$-dependence. An action $a \in \mathcal{A}$ is called $\epsilon$\emph{-dependent} of actions $a_{1:n}=\{a_1,\dots,a_n\} \in \mathcal{A}$ with respect to $\mathcal{F}$ if any pair of functions $f,\tilde{f} \in \mathcal{F}$ satisfying $\sqrt{\sum_{i=1}^n (f(a_i)-\tilde{f}(a_i))^2}\leq \epsilon$ also satisfies $f(a)-\tilde{f}(a) \leq \epsilon$ for some $\epsilon>0$. An action $a$ is $\epsilon$\emph{-independent} of $a_{1:n}$ if it is not $\epsilon$-dependent of $a_{1:n}$. The $\epsilon$\emph{-eluder dimension} $dim_E(\mathcal{F},\epsilon)$ is the length of the longest sequence of elements in $\mathcal{A}$, such that for some $\epsilon' \geq \epsilon$, every element is $\epsilon'$-independent of its predecessors. 

Informally, the eluder dimension is a measure of the `wigglyness' 
of the functions in $\mathcal{F}$, as it quantifies how long a sequence of actions may be such that at each action, there exist two functions in $\mathcal{F}$ that take well-separated values, but have similar (enough) values for all actions taken previously. We will later show that the more Lipschitz derivatives the functions in a function class have, the smaller  its eluder dimension is.

Second, we introduce a \emph{ball-width function} $\beta^*_n$. This ball-width function defines the size of high-probability confidence sets in the function class $\mathcal{F}$, in terms of $n$, a number of reward observations.
\citet{RussoVanRoy2014} introduce an analogous function, in their equation (8), for the case of sub-Gaussian noise.  The properties of sub-exponential distributions mean that our function is necessarily more complex, but its interpretation is the same. In particular $\beta^*_n$ depends on $N(\alpha,\mathcal{F},||\cdot||_\infty)$, the $\alpha$-covering number of the function class $\mathcal{F}$ with respect to the uniform norm, $||\cdot||_\infty$. Furthermore it depends on $\sigma^2$ and $b$, the sub-exponential parameters of the reward noise distribution, free parameters $\alpha,\delta>0$ which will be chosen to optimise the regret bound, and $\lambda$, 
which retains its interpretation as the free parameter in Equation \eqref{eq::subexpprop}. 

The ball-width function has the following form: 
\begin{align}\label{eq::bwf}
\beta^*_n(\mathcal{F}&,\delta,\alpha,\lambda) = \frac{2\alpha}{1-2\lambda\sigma^2}\times \\
\bigg[&\frac{\log(N(\alpha,\mathcal{F},||\cdot||_\infty)/\delta)}{2\lambda\alpha} 
+ n (4C+\alpha)(1-\lambda\sigma^2) \nonumber \\  &+ \sum_{i \leq \lfloor n_0 \rfloor}\sqrt{2\sigma^2\log(4i^2/\delta)} 
+ \sum_{i \geq \lceil n_0 \rceil}^n 2b\log(4i^2/\delta)\bigg],\nonumber 
\end{align} where $n_0=\sqrt{\frac{\delta}{4}\exp\frac{\sigma^2}{2b^2}}$.  

Together, the eluder dimension and ball-width function characterise a bound on the Bayesian regret of TS applied to the general bandit problem with reward function drawn from $\mathcal{F}$ and actions selected from $\mathcal{A}$. This bound is given in the following theorem.

\begin{theorem} \label{prop::generaleluderregret}
Consider Thompson sampling with prior $p_0$ on a function class $\mathcal{F}$ applied to the bandit problem $(\mathcal{A},f_0,p_\eta)$ where the reward function $f_0$ is drawn from $p_0$, all functions $f \in \mathcal{F}$ are $f:\mathcal{A} \rightarrow [0,C]$ for some $C>0$, and the reward noise distribution $p_\eta$ is $(\sigma^2,b)$-sub-exponential.
For all problem horizons $T \in \mathbb{N}$, nonincreasing functions $\kappa:\mathbb{N} \rightarrow \mathbb{R}_+$, and parameters $\alpha>0, \delta\leq 1/(2T),$ and  $|\lambda|\leq (2Cb)^{-1}$, it is the case that \begin{align}
BR(T) &\leq T\kappa(T)+ (dim_E(\mathcal{F},\kappa(T))+1)C \nonumber \\
&+ 4\sqrt{dim_E(\mathcal{F},\kappa(T))\beta^*_T(\mathcal{F},\alpha,\delta,\lambda)T}. \label{eq::generalbayesregret}
\end{align}
\end{theorem}

The bound \eqref{eq::generalbayesregret} is useful because it characterises the regret in terms of the eluder dimension and ball-width function of the function class $\mathcal{F}$. Each of these may be bounded in terms of $T$ based on the properties of $\mathcal{F}$. Through judicious choice of $\kappa, \alpha$, and $\delta$ as functions of $T$, we can derive regret bound expressions which are sublinear in $T$. We will do so in Section \ref{sec::SpecificFunClasses}.

As mentioned previously, Proposition 10 of \cite{RussoVanRoy2014} constructs a similar bound to \eqref{eq::generalbayesregret}. The material difference between the bounds is that in \cite{RussoVanRoy2014} $\kappa(T)$ is effectively fixed to $T^{-1}$, which unnecessarily constrains the results which can be obtained for specific function classes. By allowing for other choices of $\kappa(T)$ we have greater flexibility and can achieve tighter bounds.

In the supplementary material we provide a proof of Theorem \ref{prop::generaleluderregret}. Central to the proof is a decomposition of the Bayesian regret of TS in terms of the widths of a sequence of high probability confidence sets for $f_0$. These sets are centred on a least squares estimator of the reward function. Crucially, their widths can be written in terms of the ball-width function and eluder dimension regardless of whether the estimator itself has a convenient analytical form.

We proceed, in the following section, to specify the bound \eqref{eq::generalbayesregret} in the settings where $\mathcal{F}$ is the class of functions with $M \in \mathbb{N}$ Lipschitz derivatives. In \cite{RussoVanRoy2014}, the analogue of \eqref{eq::generalbayesregret} is extended only to (generalised) linear function classes. Our results are therefore substantially more general, since we consider non-parametric function classes, which include the (generalised) linear classes as special cases. Nevertheless, in the supplementary material, we demonstrate that our results for sub-exponential noise can explicitly be extended to these (generalised) linear function classes, with no increase in the order of the regret bound.

\section{Bounds for Smoother-than-Lipschitz Function Classes} \label{sec::SpecificFunClasses}

In this section we consider the specification of the general result to classes of functions with Lipschitz derivatives. For any $C,L>0$ and $M \in \mathbb{N}$, we define $\mathcal{F}_{C,M,L}$ as the class of $C$-bounded functions, $f:[0,1] \rightarrow [0,C]$, with $M$ $L$-Lipschitz smooth derivatives. Functions in $\mathcal{F}_{C,M,L}$ satisfy \begin{equation*}
|f^{(m)}(a)-f^{(m)}(a')| \leq L|a-a'|, \enspace  \forall a,a' \in [0,1], \label{eq::lipfunclass}
\end{equation*} for each $m \leq M$. Note that when $M=0$ this is simply the class of bounded Lipschitz functions. 

For larger $M$, including $M=\infty$, all polynomial functions are trivially included within an $\mathcal{F}_{C,M,L}$, as are appropriately weighted combinations of sufficiently smooth basis functions. Functions sampled from GPs with smooth kernels can also be shown to be members of $\mathcal{F}_{C,M,L}$, since the derivative of a GP is also a GP \cite[Section 9.4]{RasmussenWilliams2006}. We note also that each $\mathcal{F}_{C,M,L}$ may be represented as a ball within a corresponding Sobolev space, and some readers may find it instructive to think of this interpretation. 

\subsection{Regret Upper Bound}

Our main result, below, is a bound on the Bayesian regret of TS applied where $f_0$ is drawn from a prior on $\mathcal{F}_{C,M,L}$. 

\begin{theorem} \label{thm::LipschitzBayesReg}
Consider Thompson sampling with prior $p_0$ applied to the bandit problem $([0,1],f_0,p_\eta)$ where $f_0$ is drawn from a prior $p_0$ on $\mathcal{F}_{C,M,L}$ and $p_\eta$ is sub-exponential.
For all problem horizons $T \in \mathbb{N}$, we have that the Bayesian regret is bounded as \begin{equation}
BR(T) = O(T^{(2M^2+11M+10)/(4M^2+14M+12)}). \label{eq::lipregret}
\end{equation}
\end{theorem}

The consequence of this result is more transparent when we consider particular values of $M$. We have Bayesian regret of order $O(T^{5/6})$ when the reward function is Lipschitz and of order $O(T^{23/30})$ when it has a Lipschitz first derivative. As the number of Lipschitz derivatives $M \rightarrow \infty$ the order of the Bayesian regret approaches $O(\sqrt{T})$. We discuss these results in relation to lower bounds in Section \ref{sec::lowerbound}

\emph{Proof of Theorem \ref{thm::LipschitzBayesReg}}: The proof of Theorem \ref{thm::LipschitzBayesReg} relies on bounding the eluder dimension and ball-width function for the function class $\mathcal{F}_{C,M,L}$. The following theorem provides the necessary bound on the eluder dimension of Lipschitz function classes. 

\begin{theorem} \label{prop::Lipeluder}
For $M \in \mathbb{N}$, and $C,L,\epsilon>0$ the $\epsilon$-eluder dimension of $\mathcal{F}_{C,M,L}$ is bounded as follows,
\begin{equation}
\dim_E(\mathcal{F}_{C,M,L},\epsilon)=o((\epsilon/L)^{-1/(M+1)}). \label{eq::Lipeluder}
\end{equation}
\end{theorem}  This result is a non-trivial extension of the existing bounds on the eluder dimension of simpler function classes, and is the first bound on the eluder dimension of a non-parametric class of functions. A sketch of the proof of this theorem is given in Section \ref{sec::eluderproof}, and the full proof is given in the supplementary material.


To use Theorem \ref{prop::Lipeluder} within Theorem \ref{prop::generaleluderregret} we will be considering $\dim_E(\mathcal{F}_{C,M,L},\kappa(T))$ for a nonincreasing function $\kappa$. The effect of $M$ in \eqref{eq::Lipeluder} demonstrates that, for large $M$, the influence on the regret of $\kappa$ through the eluder dimension is minimal.

Bounding the ball-width function relies in turn on a bound on the covering number of the Lipschitz function class. The covering numbers of Lipschitz function classes were amongst the first to be discovered \citep{KolmogorovTikhomirov1961}. Specifically,  for $M \in \mathbb{N}$ and $\mathcal{F}_{C,M,L}$ as defined previously, the following is known, \begin{equation*}
\log N(\alpha,\mathcal{F}_{C,M,L},||\cdot||_\infty) = \Theta(\alpha^{-\frac{1}{M+1}}).
\end{equation*} We wish to select $\alpha$ as a function of $T$ to minimise the order of $\beta_T^*(\mathcal{F}_{C,M,L},\delta,\alpha(T),\lambda)$ with respect to $T$. Choosing $\alpha(T)=T^{-(M+1)/(M+2)}$ we have, \begin{equation}
\beta^*_T(\mathcal{F}_{C,M,L},\delta,T^{-\frac{M+1}{M+2}},\lambda) = O(T^{1/(M+2)}) \label{eq::ballwidthbest}
\end{equation} as the best available result.

We then complete the proof by using the general bound of \eqref{eq::generalbayesregret}. We choose $\kappa(T)=T^{-\frac{1}{2}\frac{2M^2+3M+2}{2M^2+7M+6}}$, and bound the eluder dimension as in \eqref{eq::Lipeluder} and ball-width function as in \eqref{eq::ballwidthbest} to achieve the stated result. $\square$

\subsection{Eluder Dimension Bound} \label{sec::eluderproof}
In this section we sketch the proof of the eluder dimension bound given as Theorem \ref{prop::Lipeluder}. To aid in this we first define a related function class: \begin{equation*}
\mathcal{G}_{C,M,L} = \bigg\{g=f-f', \forall f,f' \in \mathcal{F}_{C,M,L} \bigg\},
\end{equation*} which is the class of absolute difference functions for all pairs of functions in $\mathcal{F}_{C,M,L}$. As the eluder dimension is defined in terms of difference of functions $f,f' \in \mathcal{F}_{C,M,L}$, considering the behaviour of functions in $\mathcal{G}_{C,M,L}$ will allow us to bound the eluder dimension. Functions $g \in \mathcal{G}_{C,M,L}$ also possess $M$ Lipschitz derivatives. Specifically, we have the following result, which has its proof in the supplementary material. \begin{proposition} \label{prop::GHasLipDer}
All functions $g \in \mathcal{G}_{C,M,L}$ are $[-C,C]$-bounded and possess $M$ $2L$-Lipschitz smooth derivatives.
\end{proposition}

We may also define the eluder dimension in terms of $\mathcal{G}_{C,M,L}$, which will be useful for the proof of Theorem \ref{prop::Lipeluder}. Let $a_{1:k} \in [0,1]^k$ denote a sequence of actions $(a_1,\dots,a_k)$ and define \begin{align*}
w_k(a_{1:k},\epsilon')&= \sup_{g \in \mathcal{G}_{C,M,L}} \bigg\{g(a_k): \sqrt{{\textstyle\sum}_{i=1}^{k-1} (g(a_i))^2} \leq \epsilon' \bigg\}.
\end{align*} We then define the $\epsilon$-eluder dimension as follows: \begin{align*}
dim_E(\mathcal{F}_{C,M,L},\epsilon) &= \max_{\tau \in \mathbb{N}, \epsilon'>\epsilon} \Big\{ \tau: \exists \enspace  a_{1:\tau} \in [0,1]^\tau \text{ with } \\
&\quad \enspace w_k(a_{1:k},\epsilon') > \epsilon' \text{ for every } k \leq \tau \Big\} .
\end{align*}

Based on this definition we will sketch the proof of Theorem \ref{prop::Lipeluder} in the remainder of this section. The full proof is reserved for the supplementary material.

\emph{Sketch of Proof of Theorem \ref{prop::Lipeluder}:} The proof relies on the observation that $w_k(a_{1:k},\epsilon')>\epsilon'$ may only be satisfied if there exists a function $g \in \mathcal{G}_{C,M,L}$ which takes a relatively large value at $a_k$, i.e. with $g(a_k)> \epsilon'$, but changes rapidly enough to have relatively small absolute value at previous elements of the sequence, i.e. $\sum_{i=1}^{k-1} (g(a_i))^2 \leq (\epsilon')^2$. 

Any smooth function $g$ with $g(a)>\epsilon'$ at some $a \in [0,1]$ must have an associated region, of non-zero size, which we call $B(g) \subseteq [0,1]$ where $|g(x)| > \epsilon'/3$. The smoother $g$ is, the larger the region $B(g)$ must be. A necessary condition for satisfying $w_k(a_{1:k},\epsilon')>\epsilon'$ is that there exists a function $g \in \mathcal{G}_{C,M,L}$ with $g(a_k)> \epsilon'$ such that there are not too many among the points $a_{1:k-1}$ within $B(g)$, specifically fewer than nine (since $\sqrt{9\times(\epsilon'/3)^2}=\epsilon'$). 

It follows that a necessary condition for the $\epsilon$-eluder dimension of $\mathcal{F}_{C,M,L}$ to take value at least $\tau$ is that there exists a sequence $a_{1:\tau} \in [0,1]$ and a sequence of functions $g_1,\dots,g_\tau \in \mathcal{G}_{C,M,L}$ with $g_i(a_i)> \epsilon', \, i\leq \tau$, such that $\sum_{i=1}^{k-1} \mathbb{I}\{a_i \in B(g_k)\} <9$ for all $k \leq \tau$. We derive upper bounds on the eluder dimension by bounding the value of $\tau$ for which this necessary condition may be satisfied. This is feasible, as the size of the region $B(g)$ for any $g \in \mathcal{G}_{C,M,L}$ and $a \in [0,1]$ such that $g(a)>\epsilon'$ may be related to the smoothness of the class $\mathcal{G}_{C,M,L}$ and the largest value of $\tau$ such that the necessary condition can be satisfied may be related to the size of the $B$ regions.

For each choice of $M$ and an $a \in [0,1]$ we can identify a function $h_{M,a} \in \mathcal{G}_{C,M,L}$ which satisfies $h_{M,a}(a)>\epsilon'$ but minimises the size of $B_a$, i.e.
\begin{displaymath}
h_{M,a} \in \argmin_{h \in \mathcal{G}_{C,M,L}:h(a)>\epsilon'}\int \mathbb{I}\{|h(x)|\geq \epsilon'/3\}dx,\end{displaymath} and the minimising values \begin{displaymath}
B^*_{M,a} = \min_{h \in \mathcal{G}_{C,M,L}:h(a)>\epsilon'}\int \mathbb{I}\{|h(x)|\geq \epsilon'/3\}dx.
\end{displaymath}

The functions $h_{M,a}$ can be shown to be characterised by having zeros of their derivatives at specific locations. In particular, odd ordered derivatives should have zeros at $a$ and the points where $h_{M,a}(x)=-\epsilon/3$ and even ordered derivatives should have zeros at points where $h_{M,a}(x)=\epsilon/3$. Allowing the highest order derivative to be linear subject to these conditions ensures the region $B_a(h_{M,a})$ is as small as possible. Figure \ref{fig::eluderplots} illustrates functions $h_{a,0},h_{a,1}, h_{a,2}$ and their first derivatives. We can see the increasing width of $B^*_{a,M}$ as $M$ increases.

The minimising values $B^*_{M,a}$ are shown to be $o((\epsilon/L)^{1/(M+1)})$. In turn, this means that if there is a sequence of $\tau$ points $a_{1:\tau}$ with $\tau=o((\epsilon/L)^{-1/(M+1)})$ placed in $[0,1]$, it is impossible to satisfy $w_k(a_{1:k},\epsilon')$ for every $k \leq \tau$. By definition the eluder-dimension may then be bounded as $o((\epsilon/L)^{1/(M+1)})$.

\begin{figure*} [htbp]
\centering
\includegraphics[width=\textwidth,height=5cm]{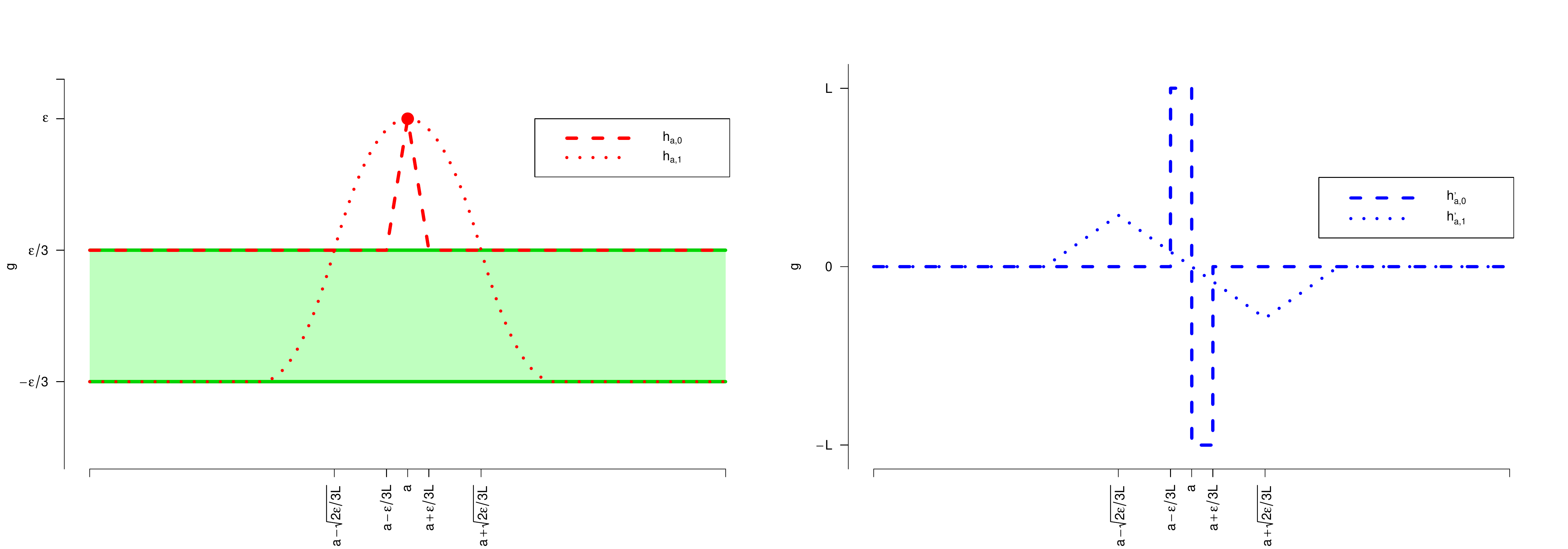}
\caption{This figure displays functions $g \in \mathcal{G}^*_{C,M,L}(a)$ for $M=0$ and $M=1$. These functions take value greater than $\epsilon$ at $a$, which is well separated from 0 and 1. The functions then decrease on the left and right in to the interval $[-\epsilon/3,\epsilon/3]$ at the quickest rate possible for functions in $\mathcal{G}_{C,M,L}$.}\label{fig::eluderplots}
\end{figure*}

\subsection{Regret Lower Bounds} \label{sec::lowerbound}
The following theorem, a restatement of Theorem 1 of \cite{BubeckEtAl2011}, gives a lower bound on the regret of any algorithm for the CAB with a Lipschitz reward function. It is an adaptation of the stronger results in \cite{Kleinberg2005,Kleinberg2008,Bubeck2011} which apply to bandits on metric spaces. For ease of exposition, and following convention, we will assume in the remainder, without loss of generality, that the bounding constant is $C=1$.

\begin{theorem} \label{thm::standard_lip_lower_bound}
Let \texttt{ALG} be any algorithm for Lipschitz continuum armed bandits with time horizon $T$, and Lipschitz constant $L$. Let $M=0$, i.e. the Lipschitz condition apply only to the reward function, not its derivatives. There exists a problem instance $\mathcal{I}=\mathcal{I}(x^*,\epsilon)$ for some $x^* \in [0,1]$ and $\epsilon>0$ such that \begin{displaymath}
\mathbb{E}(R(T)|\mathcal{I}) \geq \Omega(L^{1/3}T^{2/3}).
\end{displaymath}
\end{theorem}
The proof of the Theorem relies on the construction of a particularly challenging CAB instance $\mathcal{I}(x^*,\delta)$ with reward function $\mu$ where \begin{equation}
\mu(x) = \begin{cases}
&0.5, \quad  \text{for } x: |x-x^*| > \delta/L, \\
&0.5+\delta - L|x-x^*|, \enspace \text{otherwise}.
\end{cases} \label{eq::worstcaserewardfunc}
\end{equation} Theorem \ref{thm::standard_lip_lower_bound} does not apply for $M>0$. This is because the reward function $\mu$ defined as in \eqref{eq::worstcaserewardfunc} used to define the worst-case problem instance, does not have a Lipschitz first derivative and thus is not a valid reward function for the problem class being considered. 

In the theorem below, we give an $M$-dependent lower bound on regret, for CABs whose reward functions have $M \geq 0$ Lipschitz derivatives.

\begin{theorem} \label{thm::ext_lower_bound}
Let \texttt{ALG} be any algorithm for the CAB problem with reward function in $\mathcal{F}_{C,M,L}$. There exists a problem instance $\mathcal{I}=\mathcal{I}(x^*,\delta)$ for some $x^* \in [0,1]$ and $\delta>0$ such that \begin{displaymath}
\mathbb{E}(R(T)|\mathcal{I}) \geq \Omega(T^{(M+2)/(2M+3)}).
\end{displaymath}
\end{theorem}

The proof of this theorem is provided in the supplementary material. 

\subsection{Comparing Upper and Lower Bounds}

Firstly, we notice that for $M=\infty$, the upper and lower bounds match up to a constant, in that they are both order $\sqrt{T}$. This implies that exact TS is an order-optimal algorithm for CAB problems with reward function drawn from a prior on (any subset of) $\mathcal{F}_{C,\infty,L}$. This is a more general result than those presented in \cite{RussoVanRoy2014}, as they had similar results only for special cases within $\mathcal{F}_{C,\infty,L}$ - namely (generalised) linear reward functions and reward functions modelled as samples from Gaussian processes. Further, we even present a marginal improvement in those cases, as we remove a mutliplicative $\log(T)$ factor from the upper bounds.

Interestingly, for finite $M$, the bounds do not match. For instance, with $M=0$ the upper bound has order $O(T^{5/6})$ and the lower bound has order $\Omega(T^{2/3})$. Generally speaking there is a gap of order $T^{({3M+2})/({4M^2+14M+12})}$ between the bounds for finite $M$. This raises an interesting open question: {are the eluder-dimension based bounds simply not tight for finite $M$, or is TS inherently suboptimal?}

There would seem to be some credence to both arguments. If we consider the nature of algorithms which do achieve order optimal bounds for the Lipschitz bandit problem, such as the Zooming algorithm of \cite{Kleinberg2005}, we notice that they generally employ an adaptive discretisation component. That is to say, they limit the actions available to the algorithm to some set $\mathcal{A}_t \subset \mathcal{A}$ in each round $t \in \{1,\dots,T\}$, and in doing so force a certain level of exploration. It could be that the TS algorithm analysed here which has access to the entire action set $\mathcal{A}$ somehow carries a greater risk of conducting insufficient exploration. 

On the other hand it is possible that the true performance of the TS approach analysed here does in fact match the lower bound, and analysis of \cite{RussoVanRoy2014} which we have adapted to this setting is too loose in this framework. The contribution of the covering number term to the overall order for instance in the $M=0$ setting is $T^{1/4}$ and the $\sqrt{T}$ factor from the least squares analysis is also unavoidable. Thus, even with a $\kappa(T)$-eluder dimension of $O(1)$ the resulting bound would be suboptimal compared with the $\Omega(T^{2/3})$ lower bound. Inspection of the proof suggests that while this technique is highly versatile, it would not be possible to adapt it to achieve an optimal order bound in CAB problems whose reward function is drawn from $\mathcal{F}_{C,M,L}$, with finite $M$. 

\section{Conclusion} \label{sec::Lipbanconc}

This work extends the understanding of Thompson Sampling for stochastic bandit problems. The results are bounds on the Bayesian regret of Thompson Sampling for continuum-armed bandits where the reward function possesses $M$ Lipschitz derivatives and where the reward noise is sub-exponential. We achieved these results by extending the application of the eluder dimension technique of \cite{RussoVanRoy2014} which allows the Bayesian regret of TS to be bounded in terms of the complexity of the reward function class.

Our results represent a substantial advance on the generality of existing performance guarantees available for TS. While previous results have focussed on $d$-dimensionally parametrised functions or Gaussian process priors only, our framework captures TS based on non-parametric priors over the reward function class. As such our results are applicable in much broader settings where only limited assumptions about the reward function are possible. 

While exact sampling from the posterior distributions on which our analysis is based may be challenging, these fundamental results are useful in two regards. They provide a useful benchmarking tool for subsequent analyses, and generally inform us as to how the smoothness properties of the reward function class are likely to impact the performance of TS.

Finally, our work raises interesting open questions around the analysis of non-parametric TS. Firstly, whether the gap between the upper and lower regret bounds for finite $M$ is a feature of the eluder-dimension based analysis (i.e. it can be improved) or of TS itself (i.e. it is inherent and unavoidable). Secondly, to what extent this performance may be recovered by approximate TS algorithms, which are popular and often necessary for complex problems.

\bibliographystyle{apalike} 
\bibliography{References}

\begin{thebibliography}{}

\bibitem[Agrawal, 1995]{Agrawal1995}
Agrawal, R. (1995).
\newblock The continuum-armed bandit problem.
\newblock {\em SIAM Journal on Control and Optimization}, 33:1926--1951.

\bibitem[Agrawal and Goyal, 2012]{Agrawal2012}
Agrawal, S. and Goyal, N. (2012).
\newblock Analysis of thompson sampling for the multi-armed bandit problem.
\newblock In {\em Conference on Learning Theory}, pages 39--1.

\bibitem[Auer et~al., 2002a]{AuerEtAl2002}
Auer, P., Cesa-Bianchi, N., and Fischer, P. (2002a).
\newblock {Finite-Time Analysis of the Multiarmed Bandit Problem}.
\newblock {\em Machine Learning}, 47(2-3):235--256.

\bibitem[Auer et~al., 2002b]{AuerEtAl2002non}
Auer, P., Cesa-Bianchi, N., Freund, Y., and Schapire, R.~E. (2002b).
\newblock The nonstochastic multiarmed bandit problem.
\newblock {\em SIAM journal on computing}, 32(1):48--77.

\bibitem[Bubeck et~al., 2011a]{Bubeck2011}
Bubeck, S., Munos, R., Stoltz, G., and Szepesv\'{a}ri, C. (2011a).
\newblock {${\mathcal X}$}-armed bandits.
\newblock {\em J.\ Mach.\ Learn.\ Res.}, 12:1655--1695.

\bibitem[Bubeck et~al., 2011b]{BubeckEtAl2011}
Bubeck, S., Stoltz, G., and Yu, J.~Y. (2011b).
\newblock Lipschitz bandits without the lipschitz constant.
\newblock In {\em International Conference on Algorithmic Learning Theory},
  pages 144--158. Springer.

\bibitem[Chapelle and Li, 2011]{ChapelleLi2011}
Chapelle, O. and Li, L. (2011).
\newblock An empirical evaluation of thompson sampling.
\newblock In {\em Advances in neural information processing systems}, pages
  2249--2257.

\bibitem[Kaufmann et~al., 2012]{Kaufmann2012thompson}
Kaufmann, E., Korda, N., and Munos, R. (2012).
\newblock Thompson sampling: An asymptotically optimal finite-time analysis.
\newblock In {\em International Conference on Algorithmic Learning Theory},
  pages 199--213. Springer.

\bibitem[Kleinberg et~al., 2019]{KleinbergEtAl2019}
Kleinberg, R., Slivkins, A., and Upfal, E. (2019).
\newblock Bandits and experts in metric spaces.
\newblock {\em Journal of the ACM (JACM)}, 66(4):30.

\bibitem[Kleinberg, 2005]{Kleinberg2005}
Kleinberg, R.~D. (2005).
\newblock Nearly tight bounds for the continuum-armed bandit problem.
\newblock In {\em NeurIPS}, pages 697--704.

\bibitem[Kleinberg et~al., 2008]{Kleinberg2008}
Kleinberg, R.~D., Slivkins, A., and Upfal, E. (2008).
\newblock Multi-armed bandits in metric spaces.
\newblock In {\em Proc.\ 40th Annu.\ ACM Symp.\ on Theory of Computing}, pages
  681--690.

\bibitem[Kolmogorov and Tikhomirov, 1961]{KolmogorovTikhomirov1961}
Kolmogorov, A.~N. and Tikhomirov, V.~M. (1961).
\newblock $\epsilon$-entropy and $\epsilon$-capacity of sets in function
  spaces.
\newblock {\em Translations of the American Mathematical Society}, 17:277--364.

\bibitem[Komiyama et~al., 2015]{Komiyama2015}
Komiyama, J., Honda, J., and Nakagawa, H. (2015).
\newblock Optimal regret analysis of thompson sampling in stochastic
  multi-armed bandit problem with multiple plays.

\bibitem[Korda et~al., 2013]{Korda2013}
Korda, N., Kaufmann, E., and Munos, R. (2013).
\newblock Thompson sampling for 1-dimensional exponential family bandits.
\newblock In {\em NeurIPS}, pages 1448--1456.

\bibitem[Lai and Robbins, 1985]{LaiRobbins1985}
Lai, T.~L. and Robbins, H. (1985).
\newblock Asymptotically efficient adaptive allocation rules.
\newblock {\em Advances in Applied Mathematics}, 6(1):4--22.

\bibitem[Lattimore and Szepesv{\'a}ri, 2018]{Lattimore2018}
Lattimore, T. and Szepesv{\'a}ri, C. (2018).
\newblock Bandit algorithms.
\newblock {\em preprint}.

\bibitem[Lu et~al., 2019]{LuEtAl2019}
Lu, S., Wang, G., Hu, Y., and Zhang, L. (2019).
\newblock Optimal algorithms for lipschitz bandits with heavy-tailed rewards.
\newblock In {\em International Conference on Machine Learning}, pages
  4154--4163.

\bibitem[Lu and Van~Roy, 2017]{LuVanRoy2017}
Lu, X. and Van~Roy, B. (2017).
\newblock Ensemble sampling.
\newblock In {\em Advances in neural information processing systems}, pages
  3258--3266.

\bibitem[May et~al., 2012]{May2012}
May, B.~C., Korda, N., Lee, A., and Leslie, D.~S. (2012).
\newblock Optimistic {B}ayesian sampling in contextual-bandit problems.
\newblock {\em J.\ Mach.\ Learn.\ Res.}, 13:2069--2106.

\bibitem[Osband and Van~Roy, 2014]{OsbandVanRoy2014}
Osband, I. and Van~Roy, B. (2014).
\newblock Model-based reinforcement learning and the eluder dimension.
\newblock In {\em Advances in Neural Information Processing Systems}, pages
  1466--1474.

\bibitem[Phan et~al., 2019]{PhanEtAl2019}
Phan, M., Abbasi-Yadkori, Y., and Domke, J. (2019).
\newblock Thompson sampling and approximate inference.
\newblock {\em arXiv preprint arXiv:1908.04970}.

\bibitem[Russo and Van~Roy, 2014]{RussoVanRoy2014}
Russo, D. and Van~Roy, B. (2014).
\newblock Learning to optimize via posterior sampling.
\newblock {\em Math.\ Oper.\ Res.}, 39:1221--1243.

\bibitem[Russo et~al., 2018]{Russo2018}
Russo, D., Van~Roy, B., Kazerouni, A., Osband, I., and Wen, Z. (2018).
\newblock A tutorial on {T}hompson sampling.
\newblock {\em Found.\ Trends Mach.\ Learn.}, 11:1--96.

\bibitem[Shahriari et~al., 2016]{ShahriariEtAl2016}
Shahriari, B., Swersky, K., Wang, Z., Adams, R.~P., and De~Freitas, N. (2016).
\newblock Taking the human out of the loop: A review of bayesian optimization.
\newblock {\em Proceedings of the IEEE}, 104(1):148--175.

\bibitem[Slivkins, 2019]{Slivkins2019}
Slivkins, A. (2019).
\newblock Introduction to multi-armed bandits.
\newblock {\em arXiv preprint arXiv:1904.07272}.

\bibitem[Srinivas et~al., 2010]{SrinivasEtAl2009}
Srinivas, N., Krause, A., Kakade, S.~M., and Seeger, M.~W. (2010).
\newblock Gaussian process optimization in the bandit setting: No regret and
  experimental design.
\newblock In {\em Proceedings of the 27th International Conference on Machine
  Learning}.

\bibitem[Srinivas et~al., 2012]{SrinivasEtAl2012}
Srinivas, N., Krause, A., Kakade, S.~M., and Seeger, M.~W. (2012).
\newblock Information-theoretic regret bounds for gaussian process optimization
  in the bandit setting.
\newblock {\em IEEE Transactions on Information Theory}, 58(5):3250--3265.

\bibitem[Thompson, 1933]{Thompson1933}
Thompson, W.~R. (1933).
\newblock On the likelihood that one unknown probability exceeds another in
  view of the evidence of two samples.
\newblock {\em Biometrika}, 25(3/4):285--294.

\bibitem[Urteaga and Wiggins, 2018a]{UrteagaWiggins2018a}
Urteaga, I. and Wiggins, C. (2018a).
\newblock Variational inference for the multi-armed contextual bandit.
\newblock In {\em International Conference on Artificial Intelligence and
  Statistics}, pages 698--706.

\bibitem[Urteaga and Wiggins, 2018b]{UrteagaWiggins2018b}
Urteaga, I. and Wiggins, C.~H. (2018b).
\newblock (sequential) importance sampling bandits.
\newblock {\em arXiv preprint arXiv:1808.02933}.

\bibitem[Wang and Chen, 2018]{Wang2018}
Wang, S. and Chen, W. (2018).
\newblock Thompson sampling for combinatorial semi-bandits.
\newblock In {\em International Conference on Machine Learning}, pages
  5101--5109.

\bibitem[Williams and Rasmussen, 2006]{RasmussenWilliams2006}
Williams, C.~K. and Rasmussen, C.~E. (2006).
\newblock {\em Gaussian processes for machine learning}, volume~2.
\newblock MIT Press Cambridge, MA.

\end{thebibliography}

\appendix
\onecolumn

\section{Proof of Theorem \ref{prop::generaleluderregret}: General regret bound} \label{sec::Thm1proof} \label{sec::Lip_proofs}
In this section, we provide a proof of the general eluder-dimension-based bound on the Bayesian regret of Thompson Sampling (TS), and the proofs of technical lemmas which support the main proof. Certain results and definitions from the main paper will be restated for convenience.  


\textbf{Theorem \ref{prop::generaleluderregret}} \emph{
Consider Thompson sampling with prior $p_0$ on a function class $\mathcal{F}$ applied to the bandit problem $(\mathcal{A},f_0,p_\eta)$ where the reward function $f_0$ is drawn from a $p_0$, all functions $f \in \mathcal{F}$ are $f:\mathcal{A} \rightarrow [0,C]$ for some $C>0$, and the reward noise distribution $p_\eta$ is $(\sigma^2,b)$-sub-exponential.
For all problem horizons $T \in \mathbb{N}$, nonincreasing functions $\kappa:\mathbb{N} \rightarrow \mathbb{R}_+$, and parameters $\alpha>0, \delta\leq 1/(2T),$ and  $|\lambda|\leq (2Cb)^{-1}$, it is the case that} \begin{align*}
BR(T) &\leq T\kappa(T)+ (dim_E(\mathcal{F},\kappa(T))+1)C + 4\sqrt{dim_E(\mathcal{F},\kappa(T))\beta^*_T(\mathcal{F},\alpha,\delta,\lambda)T}. 
\end{align*}

We begin the proof with the following martingale concentration result, an extension of Lemma 3 of \cite{RussoVanRoy2014} (which holds for sub-Gaussian noise). The result below says that with high probability, for any function $f: \mathcal{A} \rightarrow \mathbb{R}$, its squared error $L_{2,t}(f)= \sum_{i=1}^{t-1}(f(A_i)-R_i)^2$ is lower bounded. In particular, we say that with high probability the squared error of $f$ will not fall below the sum of the squared error of the true reward generating function, $f_0$, and a measure of the distance between $f$ and $f_0$, by more than a fixed constant.

\begin{lemma} \label{lem::subexpmartingale}
For any action sequence $A_1,A_2,\dots \in \mathcal{A}$, inducing $(\sigma^2,b)$-sub-exponential reward observations $R_1, R_2, \dots$ and any function $f: \mathcal{A} \rightarrow \mathbb{R}$, we have \begin{equation}
\mathbb{P}\bigg( L_{2,n+1}(f) \geq  L_{2,n+1}(f_0) + (1-2\lambda\sigma^2)\sum_{i=1}^n (f(A_i)-f_0(A_i))^2  - \frac{\log(1/\delta)}{\lambda} , \enspace \forall n \in \mathbb{N} \bigg) \geq 1-\delta,
\end{equation} for all $\lambda$ with $|\lambda|\leq (2Cb)^{-1}$.
\end{lemma}

\begin{proof}
The proof is based on the sub-exponential property of the reward noise. \label{proof:lemma1}
First consider arbitrary random variables $\{Z_i\}_{i \in \mathbb{N}}$ adapted to a filtration $\{\mathcal{H}_i\}_{i\in\mathbb{N}}$. Assume that $\mathbb{E}(e^{\lambda Z_i})$ is finite for $\lambda \geq 0$, and define the conditional mean $\mu_i=\mathbb{E}(Z_i|\mathcal{H}_{i-1})$ and conditional cumulant generating function of the centred random variable $[Z_i-\mu_i]$ as $\psi_i(\lambda)=\log\mathbb{E}(\exp(\lambda[Z_i-\mu_i])|\mathcal{H}_{i-1})$. 
By Lemmas 6 and 7 of \cite{RussoVanRoy2014}, 
for all $x \geq 0$, and $\lambda \geq 0$, \begin{equation}
\mathbb{P}\bigg( \sum_{i=1}^n \lambda Z_i \leq x +  \sum_{i=1}^n [\lambda \mu_i + \psi_i(\lambda)], \enspace \forall n \in \mathbb{N} \bigg) \geq 1-e^{-x}. \label{eq::Lem7}
\end{equation}

Now consider $Z_i$ defined in terms of squared error terms of both the true function $f_0$ and an arbitrary function $f$:
\begin{align*}
Z_i &= (f_0(A_i)-R_i)^2 - (f(A_i)-R_i)^2 \nonumber \\
    &= -(f(A_i)-f_0(A_i))^2 + 2(f(A_i)-f_0(A_i))\eta_i,
\end{align*}
where we have used that $R_i=f_0(A_i)+\eta_i$.
The conditional mean and conditional cumulant generating function of these $Z_i$ are
\begin{align}
\mu_i &= \mathbb{E}(Z_i|\mathcal{H}_{i-1}) = -(f(A_i)-f_0(A_i))^2, \\
\psi_i(\lambda) &= \log\mathbb{E}(\exp(\lambda[Z_i-\mu_i]|\mathcal{H}_{i-1}) = \log\mathbb{E}(\exp(2\lambda(f(A_i)-f_0(A_i))\epsilon_i)|\mathcal{H}_{i-1}). \label{eq::cumgen}
\end{align}
Therefore, by the sub-exponentiality assumption we have that \begin{displaymath}
\psi_i(\lambda)  \leq \frac{4\lambda^2(f(A_i)-f_0(A_i))^2\sigma^2}{2}, \quad \text{for } |\lambda| \leq (2Cb)^{-1},
\end{displaymath}
where the bound on $\lambda$ results from the bound in absolute value of both $f_0$ and $f$.

Noting that $\sum_{i=1}^n Z_i= L_{2,n+1}(f_0)-L_{2,n+1}(f)$, use \eqref{eq::Lem7}, \eqref{eq::cumgen}, and set 
$x=\log(1/\delta)$, to find \begin{equation}
\mathbb{P}\bigg( L_{2,n+1}(f) \geq  L_{2,n+1}(f_0) + (1-2\lambda\sigma^2)\sum_{i=1}^n (f(A_i)-f_0(A_i))^2  - \frac{\log(1/\delta)}{\lambda} , \enspace \forall n \in \mathbb{N} \bigg) \geq 1-\delta,
\end{equation} for all $\lambda$ with $|\lambda|\leq (2Cb)^{-1}$, completing the proof.
\end{proof}


Lemma \ref{lem::subexpmartingale} allows us to construct high-probability confidence sets for the true reward function, $f_0$. These sets are defined with respect to the least squares estimate of $f_0$, i.e.\ the function $\hat{f}^{LS}_t = \argmin_{f \in \mathcal{F}} L_{2,t}(f)$ with minimal squared error, in reference to the observed rewards. The following lemma gives the definition and high-confidence property of said confidence sets.

\begin{lemma} \label{prop::confball}
For all $\delta>0$, $\alpha>0$, $|\lambda| \leq (2Cb)^{-1}$, $n\in\mathbb{N}$ and $\{A_1,\dots A_n\} \in \mathcal{A}^n$, define the confidence set
\begin{equation}
\mathcal{F}_n = \bigg\{f \in \mathcal{F}: \sum_{i=1}^n (\hat{f}_n^{LS}(A_i)-f(A_i))^2 \leq \beta^*_n(\mathcal{F},\delta,\alpha,\lambda) \bigg\}. \label{eq::lssets}
\end{equation}
It is the case that \begin{displaymath}
\mathbb{P}\bigg(f_0 \in \bigcap_{n=1}^\infty \mathcal{F}_n \bigg) \geq 1-2\delta.
\end{displaymath} 
\end{lemma}

\begin{proof}
\label{proof::lemma2} Let $\mathcal{F}^\alpha$ be an $\alpha$-covering of $\mathcal{F}$ of size $N(\alpha,\mathcal{F},||\cdot||_\infty)$, in the sense that for any $f \in \mathcal{F}$ there is an $f^\alpha \in \mathcal{F}^\alpha$ such that $||f^\alpha-f||_\infty \leq \alpha$. By Lemma \ref{lem::subexpmartingale} and a union bound over $\mathcal{F}^\alpha$ we have, with probability at least $1-\delta$, \begin{displaymath}
L_{2,n+1}(f^\alpha)-L_{2,n+1}(f_0) \geq  (1-2\lambda\sigma^2)\sum_{i=1}^n (f^\alpha(A_i)-f_0(A_i))^2  - \frac{1}{\lambda}\log\bigg(\frac{|\mathcal{F}^\alpha|}{\delta}\bigg) , \enspace \forall n \in \mathbb{N}, \enspace \forall f^\alpha \in \mathcal{F}^\alpha.
\end{displaymath} Then, by simple addition and subtraction, we have for any $f \in \mathcal{F}$, with probability at least $1-\delta$, \begin{align*}
&L_{2,n+1}(f)-L_{2,n+1}(f_0) \geq (1-2\lambda\sigma^2)\sum_{i=1}^n (f(A_i)-f_0(A_i))^2 - \frac{1}{\lambda}\log\bigg( \frac{|\mathcal{F}^\alpha|}{\delta}\bigg) \\
& \quad + L_{2,n+1}(f)-L_{2,n+1}(f^\alpha)  + (1-2\lambda\sigma^2)\sum_{i=1}^n \left\{(f^\alpha(A_i)-f_0(A_i))^2-(f(A_i)-f_0(A_i))^2\right\} , \enspace \forall n \in \mathbb{N}, \enspace \forall f^\alpha \in \mathcal{F}^\alpha.
\end{align*}
The probability this statement holds for all $f_\alpha$ is no larger than the probability it holds for the minimising $f_\alpha$. So, for arbitrary $f\in\mathcal{F}$, with probability at least $1-\delta$,
\begin{align*}
&L_{2,n+1}(f)-L_{2,n+1}(f_0) \geq (1-2\lambda\sigma^2)\sum_{i=1}^n (f(A_i)-f_0(A_i))^2 - \frac{1}{\lambda}\log\bigg( \frac{|\mathcal{F}^\alpha|}{\delta}\bigg) \\
&\quad + \min_{f^\alpha \in \mathcal{F}^\alpha} \left[L_{2,n+1}(f)-L_{2,n+1}(f^\alpha)  + (1-2\lambda\sigma^2)\sum_{i=1}^n \left\{(f^\alpha(A_i)-f_0(A_i))^2-(f(A_i)-f_0(A_i))^2\right\} \right] , \enspace \forall n \in \mathbb{N}.
\end{align*}
We refer to the term in the second line of this expression as the \emph{discretisation error}. Lemma \ref{lem::discretisationerror} gives a probability $1-\delta$ bound of $2\alpha n(4C+\alpha)(1-\lambda\sigma^2) + 2\alpha\sum_{i \leq \lfloor n_0 \rfloor}\sqrt{2\sigma^2\log(4i^2/\delta)} + 2\alpha\sum_{i \geq \lceil n_0 \rceil}^n 2b\log(4i^2/\delta)$ on the absolute value of the discretisation error, where $n_0=\sqrt{\frac{\delta}{4}\exp\frac{\sigma^2}{2b^2}}$.

We now set $f$ equal to the least squares estimator, $\hat{f}^{LS}_n$. Noting that $L_{2,n+1}(\hat{f}^{LS}_n) \leq L_{2,n+1}(f_0)$, and recalling that $|\mathcal{F}^\alpha|=N(\alpha,\mathcal{F},||\cdot||_\infty)$, with probability at least $1-2\delta$ \begin{align*}
(1-2\lambda\sigma^2)\sum_{i=1}^n (\hat{f}_n^{LS}(A_i)-f_0(A_i))^2 &\leq \frac{1}{\lambda}\log\bigg( \frac{N(\alpha,\mathcal{F},||\cdot||_\infty)}{\delta}\bigg) +  2\alpha n(4C+\alpha)(1-\lambda\sigma^2)  \\
&\quad \quad + 2\alpha\sum_{i \leq \lfloor n_0 \rfloor}\sqrt{2\sigma^2\log(4i^2/\delta)} + 2\alpha\sum_{i \geq \lceil n_0 \rceil}^n 2b\log(4i^2/\delta)\enspace \forall n \in \mathbb{N}.
\end{align*}
Dividing throughout by $(1-2\lambda\sigma^2)$, and recalling the formula \eqref{eq::bwf} for $\beta^*$ and the definition \eqref{eq::lssets} of the $\mathcal{F}_n$, this shows that $\mathbb{P}\left(f_0\in \bigcap_{n=1}^\infty \mathcal{F}_n\right)\geq 1-2\delta$ as required.
\end{proof}

We now prove the discretisation error result required for the proof.

\begin{lemma} \label{lem::discretisationerror}
If $f^\alpha$ satisfies $||f-f^\alpha||_\infty \leq \alpha$, and $|\lambda|\leq (2Cb)^{-1}$, then with probability at least $1-\delta$,
\begin{align*}
\bigg| L_{2,n+1}(f)-L_{2,n+1}(f^\alpha) +& (1-2\lambda\sigma^2)\sum_{i=1}^n (f^\alpha(A_i)-f_0(A_i))^2-(f(A_i)-f_0(A_i))^2 \bigg| \\
&\leq 2\alpha n(4C+\alpha)(1-\lambda\sigma^2) + 2\alpha\sum_{i \leq \lfloor n_0 \rfloor}\sqrt{2\sigma^2\log(4i^2/\delta)} + 2\alpha\sum_{i \geq \lceil n_0 \rceil}^n 2b\log(4i^2/\delta),
\end{align*} where $n_0=\sqrt{\frac{\delta}{4}\exp\frac{\sigma^2}{2b^2}}$.
\end{lemma}

\begin{proof}
As in the proof of Lemma 8 of \cite{RussoVanRoy2014} we have \begin{align*}
|(f^\alpha(a)-f_0(a))^2-(f(a)-f_0(a))^2| &\leq 4C\alpha+\alpha^2 \\
|(R_i-f(a))^2-(R_i-f^\alpha(a))^2| &\leq 2\alpha|R_i|+2C\alpha+\alpha^2
\end{align*} for all $a \in \mathcal{A}$ and $\alpha \in [0,C]$. Then summing over time, we have that 
\begin{align*}
\bigg| L_{2,n+1}(f)-L_{2,n+1}(f^\alpha) + (1-2\lambda\sigma^2)&\sum_{i=1}^n (f^\alpha(A_i)-f_0(A_i))^2-(f(A_i)-f_0(A_i))^2 \bigg| \\
&\leq \sum_{i=1}^n (1-2\lambda\sigma^2)(4C\alpha+\alpha^2)+ 2\alpha|R_i|+2C\alpha+\alpha^2 \\
&\leq \sum_{i=1}^n (1-2\lambda\sigma^2)(4C\alpha+\alpha^2)+ 2\alpha(C+|\eta_i|)+2C\alpha+\alpha^2 \\
&=\sum_{i=1}^n 2(4C\alpha+\alpha^2)(1-\lambda\sigma^2)+2\alpha|\eta_i|.
\end{align*}
Since $\eta_i$ is $(\sigma^2,b)$-sub-exponential we have the following exponential bound \begin{displaymath}
\mathbb{P}(|\eta_i|\geq x) \leq \begin{cases}
&2\exp(-x^2/2\sigma^2) \quad \text{if } 0\leq x \leq \sigma^2/b \\
&2\exp(-x/2b) \quad\quad  \text{if } x> \sigma^2/b. \end{cases}
\end{displaymath}
Then, by the independence of reward noises, and union bound: 
\begin{align*}
&\mathbb{P}\bigg(\exists i \in \mathbb{N}: |\eta_i| \geq \sqrt{2\sigma^2\log(4i^2/\delta)}\mathbb{I}\{i:\sqrt{2\sigma^2\log(4i^2/\delta)}\leq \sigma^2/b\} \\
& \quad \quad \quad \quad \quad \quad \quad \quad \quad \quad \quad + 2b\log(4i^2/\delta)\mathbb{I}\{i:2b\log(4i^2/\delta)>\sigma^2/b\}\bigg) \\
&\leq \frac{\delta}{2}\sum_{i=1}^\infty \frac{1}{i^2} \leq \delta.
\end{align*}
Thus, with probability at least $1-\delta$, \begin{align*}
\bigg| L_{2,n+1}(f)-&L_{2,n+1}(f^\alpha) + (1-2\lambda\sigma^2)\sum_{i=1}^n (f^\alpha(A_i)-f_0(A_i))^2-(f(A_i)-f_0(A_i))^2 \bigg| \\
&\leq \sum_{i=1}^n 2(4C\alpha+\alpha^2)(1-\lambda\sigma^2) \\
&\quad \quad + 2\alpha\bigg(\sqrt{2\sigma^2\log\bigg(\frac{4i^2}{\delta}\bigg)}\mathbb{I}\bigg\{\log\bigg(\frac{4i^2}{\delta}\bigg) \leq\frac{\sigma^2}{2b^2}\bigg\} + 2b\log\bigg(\frac{4i^2}{\delta}\bigg)\mathbb{I}\bigg\{\log\bigg(\frac{4i^2}{\delta}\bigg)>\frac{\sigma^2}{2b^2}\bigg\}\bigg) \\
&= 2\alpha n(4C+\alpha)(1-\lambda\sigma^2) \\
&\quad + 2\alpha\sum_{i=1}^n \bigg(\sqrt{2\sigma^2\log\bigg(\frac{4i^2}{\delta}\bigg)}\mathbb{I}\bigg\{i \leq\sqrt{\frac{\delta}{4}\exp\frac{\sigma^2}{2b^2}}\bigg\} + 2b\log\bigg(\frac{4i^2}{\delta}\bigg)\mathbb{I}\bigg\{i >\sqrt{\frac{\delta}{4}\exp\frac{\sigma^2}{2b^2}}\bigg\}\bigg)
\end{align*} and the required result follows.
\end{proof}

The confidence sets $\{\mathcal{F}_n\}_{n=1}^\infty$ defined in Lemma \ref{prop::confball}, allow us to bound the Bayesian regret of TS. Specifically, we can decompose the Bayesian regret in terms of a notion of the width of these confidence intervals. 

By Lemma 4 of \cite{RussoVanRoy2014}, we have for all problem horizons $T \in \mathbb{N}$, that if sets $\{\mathcal{F}\}_{t=1}^T$ are such that $\inf_{f \in \mathcal{F}_t}f(a) \leq f_0(a) \leq \sup_{f \in \mathcal{F}_t} f(a)$ for all $t \leq T$ and $a \in \mathcal{A}$ with probability at least $1-1/T$ then \begin{equation}
BR(T)
\leq C+ \mathbb{E}\bigg(\sum_{t=1}^T \sup_{f \in \mathcal{F}_t} f(A_t) - \inf_{f \in \mathcal{F}_t}f(A_t) \bigg). \label{eq::BRbound}
\end{equation} It is clear from Lemmas \ref{lem::subexpmartingale} and \ref{prop::confball} that the sets defined in \eqref{eq::lssets} satisfy this property. Therefore, the proof of Theorem \ref{prop::generaleluderregret} can then be completed by bounding the widths of the confidence sets, defined as \begin{displaymath}
w_{\mathcal{F}_t}(a)=\sup_{f \in\mathcal{F}_t}f(a)-\inf_{f \in \mathcal{F}_t}f(a).
\end{displaymath} The following Lemma provides such a result by bounding the sum of the widths in terms of the $\kappa(T)$-eluder dimension, $\dim_{E}(\mathcal{F},\kappa(T))$. It is a generalisation of Lemma 5 of \cite{RussoVanRoy2014} which fixes $\kappa(t)=t^{-1}$.
\begin{lemma} \label{prop::lemma5RussoVanRoy}
If $\{\beta_t\}_{t \in \mathbb{N}}$ is a non-negative, non-decreasing sequence and $\mathcal{F}_t$ is \begin{displaymath}
\mathcal{F}_t := \bigg\{f \in \mathcal{F}: {{\textstyle\sum}_{i=1}^t(\hat{f}^{LS}_i(A_i)-f(A_i))^2} \leq {\beta_t} \bigg\}
\end{displaymath} then for all $T \in \mathbb{N}$, and nonincreasing functions $\kappa: \mathbb{N} \rightarrow \mathbb{R}_+$ \begin{equation}
\sum_{t=1}^T w_{\mathcal{F}_t}(A_t) \leq T\kappa(T)+ dim_E(\mathcal{F},\kappa(T))C +4\sqrt{dim_E(\mathcal{F},\kappa(T))\beta_TT}.
\end{equation}
\end{lemma}

\begin{proof}
The proof of Lemma \ref{prop::lemma5RussoVanRoy} depends on Proposition 8 of \cite{RussoVanRoy2014}, which tells us that the definition of $\mathcal{F}_t$ in the lemma implies that \begin{equation}\label{prop::RVRProp8}
\sum_{t=1}^T \mathbb{I}\{w_{\mathcal{F}_t}(A_t)>\epsilon\} \leq \bigg(\frac{4\beta_T}{\epsilon}+1 \bigg)\dim_E(\mathcal{F},\epsilon)
\end{equation} for all $T \in \mathbb{N}$ and $\epsilon>0$.

Now, define $w_t=w_{\mathcal{F}_t}(A_t)$ and reorder the sequence $(w_1,\dots,w_T) \rightarrow (w_{i_1},\dots,w_{i_T})$ in descending order such that $w_{i_1}\geq w_{i_2} \geq \dots \geq w_{i_T}$. We have \begin{align*}
\sum_{t=1}^T w_{\mathcal{F}_t}(A_t) &= \sum_{t=1}^T w_{i_t} \\
				&=\sum_{t=1}^T w_{i_t}\mathbb{I}\{w_{i_t}\leq \kappa(T)\} + \sum_{t=1}^T w_{i_t} \mathbb{I}\{w_{i_t}>\kappa(T)\} \\
				&\leq T\kappa(T)  + \sum_{t=1}^T w_{i_t} \mathbb{I}\{w_{i_t}>\kappa(T)\}.
\end{align*}
As a consequence of $(w_{i_1},\dots,w_{i_T})$ being arranged in descending order we have for $t \in [T]$ that $w_{i_t}>\epsilon$ $\Rightarrow$ $\sum_{k=1}^t \mathbb{I}\{w_{\mathcal{F}_k}(A_k)>\epsilon\} \geq t$. By \eqref{prop::RVRProp8}, $w_{i_t}>\epsilon$ is only possible if $t \leq \big(\frac{4\beta_T}{\epsilon}+1 \big)\dim_E(\mathcal{F},\epsilon)$. Furthermore, $\epsilon\geq \kappa(T)$ $\Rightarrow$ $\dim_E(\mathcal{F},\epsilon) \leq \dim_E(\mathcal{F},\kappa(T))$ since $\dim_E(\mathcal{F},\epsilon')$ is non-increasing in $\epsilon'$. Therefore if $w_{i_t}>\epsilon\geq \kappa(T)$ we have that $t < \big(\frac{4\beta_T}{\epsilon}+1 \big)\dim_E(\mathcal{F},\epsilon)$, i.e. $\epsilon^2 \leq \sqrt{\frac{4\beta_T\dim_E(\mathcal{F},\kappa(T))}{t-\dim_E(\mathcal{F},\kappa(T))}}$. Thus, if $w_{i_t} > \kappa(T)$ $\Rightarrow$ $w_{i,t} \leq \min(C, \sqrt{\frac{4\beta_T\dim_E(\mathcal{F},\kappa(T))}{t-\dim_E(\mathcal{F},\kappa(T))}})$, and finally \begin{align*}
\sum_{t=1}^T w_{i_t} \mathbb{I}\{w_{i_t}>\kappa(T)\} &\leq \dim_{E}(\mathcal{F},\kappa(T))C+ \sum_{t=\dim_{E}(\mathcal{F},\kappa(T))+1}^T \sqrt{\frac{4\beta_T\dim_E(\mathcal{F},\kappa(T))}{t-\dim_E(\mathcal{F},\kappa(T))}} \\
&\leq \dim_{E}(\mathcal{F},\kappa(T))C+ 2\sqrt{\beta_T\dim_E(\mathcal{F},\kappa(T))}\int_{t=0}^T \frac{1}{\sqrt{t}}dt \\
&\leq  \dim_{E}(\mathcal{F},\kappa(T))C + 4\sqrt{\beta_T\dim_E(\mathcal{F},\kappa(T))T}. \qedhere
\end{align*}
\end{proof}

The conclusions of Lemmas \ref{prop::confball} and \ref{prop::lemma5RussoVanRoy}, along with \eqref{eq::BRbound}, combine to give the bound on Bayesian regret which comprises Theorem \ref{prop::generaleluderregret}, \begin{displaymath}
BR(T) \leq T\kappa(T)+ (dim_E(\mathcal{F},\kappa(T))+1)C 
+ 4\sqrt{dim_E(\mathcal{F},\kappa(T))\beta^*_T(\mathcal{F},\alpha,\delta,\lambda)T}. 
\end{displaymath}

\newpage
\section{Further Proofs for the Eluder Dimension Bound}
In this section, we provide a proof of the bound on the eluder dimension of the function classes $\mathcal{F}_{C,M,L}$ of functions with $M \in \mathbb{N}$ Lipschitz derivatives, and the proofs of technical results which support the main proof. Again, where necessary, we will restate results and definitions from the main paper.

\subsection{Proof of Proposition \ref{prop::GHasLipDer}} 
\textbf{Proposition 1} \emph{
All functions $g \in \mathcal{G}_{C,M,L}$ are $[-C,C]$-bounded and possess $M$ $2L$-Lipschitz smooth derivatives.
}

\emph{Proof of Proposition \ref{prop::GHasLipDer}:} We have that any function $g \in \mathcal{G}_{C,M,L}$ is bounded since, $f(a) \in [0,C]$ for all $a \in [0,1]$. The Lipschitz-smoothness of the $m^{th}$ derivatives can be shown as follows. For any function $g=f-f'$ where $f,f' \in \mathcal{F}_{C,M,L}$, $m =0,\dots,M$, and pair of actions $a,a' \in [0,1]$, \begin{align*}
|g^{(m)}(a)-g^{(m)}(a')|&= |f^{(m)}(a)-{f'}^{(m)}(a)-f^{(m)}(a')+{f'}^{(m)}(a')| \\
&\leq |f^{(m)}(a)-f^{(m)}(a')|+|{f'}^{(m)}(a')-{f'}^{(m)}(a)| \\
&\leq 2L||a-a'||,
\end{align*} where the first inequality holds by the triangle inequality, and the second by the $L$-Lipschitz smoothness of the $M^{th}$ derivatives of functions in $\mathcal{F}_{C,M,L}$. $\square$ 

\subsection{Proof of Theorem \ref{prop::Lipeluder}}  
\textbf{Theorem \ref{prop::Lipeluder}} \emph{For $M \in \mathbb{N}$, and $C,L,\epsilon>0$ the $\epsilon$-eluder dimension of $\mathcal{F}_{C,M,L}$ is bounded as follows,
\begin{equation*}
\dim_E(\mathcal{F}_{C,M,L},\epsilon)=o((\epsilon/L)^{-1/(M+1)}). 
\end{equation*}}

\emph{Proof of Theorem \ref{prop::Lipeluder}:} For any $k \in \mathbb{N}$ and sequence $a_{1:k} \in [0,1]^k$,  the event $\{w_{k}(a_{1:k},\epsilon')>\epsilon'\}$ by definition implies that there exists $g \in \mathcal{G}_{C,M,L}$ such that $g(a_k)>\epsilon'$ and $\sum_{i=1}^{k-1} (g(a_i))^2 \leq (\epsilon')^2$. Conversely if for all $g \in \mathcal{G}_{C,M,L}$ the event $\{g(a_k)> \epsilon'\}$ is known to imply $\sum_{i=1}^{k-1} (g(a_i))^2 > (\epsilon')^2$, then  $w_{k}(a_{1:k},\epsilon')\leq\epsilon'$. This second idea will be central to proving Theorem \ref{prop::Lipeluder}. 

We will show that for functions $g \in \mathcal{G}_{C,M,L}$ if $g(a_k)>\epsilon'$ then $g^2(b)>(\epsilon')^2/9$ for all $b$ in a certain region around $a_k$. This is a consequence of functions in $\mathcal{G}_{C,M,L}$ having $M$ smooth derivatives. If $g$ takes value greater than $\epsilon'$ at a given point, then it must take relatively large values within a certain neighbourhood of that given point. The size of this neighbourhood is a function of the level of smoothness of $g$. As $M$ increases, the size of this region where $g^2(b)> (\epsilon')^2/9$ increases. It follows that as $M$ increases, the previous actions $a_{1:k-1}$ must be increasingly far from $a_k$ for $\sum_{i=1}^{k-1} (g(a_i))^2 \leq (\epsilon')^2$ to be satisfied. Thus as $M$ increases, the eluder dimension decreases, since the condition that $\sum_{i=1}^{k-1}(g(a_i))^2 \leq (\epsilon')^2$ can only be satisfied for smaller $k$.

To be precise about this behaviour and derive the required bound on the eluder dimension, we will first lower bound the size of the neighbourhood in which $g$ must take large absolute values if $g(a)> \epsilon'$ for some $a \in [0,1]$. To aid in this we introduce the following additional notation. For a function $g: [0,1] \rightarrow[-C,C]$ define the region where it takes absolute value greater than $\epsilon/3$ as \begin{equation}
B(g) := |\{b \in [0,1]  :g(b)^2 > \epsilon^2/9\}|.
\end{equation} Then for an action $a \in [0,1]$ define the minimum size of the set such that $g^2$ must exceed $\epsilon^2/9$ if $g(a)>\epsilon$ and $g \in \mathcal{G}_{C,M,L}$ as \begin{equation}
B^*_{C,M,L}(a) := \min_{g \in \mathcal{G}_{C,M,L}: g(a)> \epsilon} B(g),
\end{equation} and the set of functions attaining this minimum as \begin{equation}
\mathcal{G}^*_{C,M,L}(a) = \argmin_{g \in \mathcal{G}_{C,M,L}: g(a)> \epsilon} B(g).
\end{equation} Bounds on $B^*_{C,M,L}(a)$, derived by identifying and considering the form of functions in $\mathcal{G}^*_{C,M,L}(a)$, will allow us to bound the eluder dimension.

We will first provide lower bounds on $B^*_{C,M,L}$ for the special cases of $M=0$ and $M=1$, and then show a general result for $M\geq 2$. In the case of $M=0$ the lower bound follows from the Lipschitz property of all functions $g \in \mathcal{G}_{C,M,L}$. We give the lower bound on $B^*_{C,0,L}(a)$ for all $a \in [0,1]$ in the following lemma. 

\begin{lemma} \label{lem::LBifMaximum}
For $a \in [0,1]$, and $C,L>0$ we have $B^*_{C,0,L}(a) \geq \frac{\epsilon}{3L}$.
\end{lemma}

\noindent \emph{Proof of Lemma \ref{lem::LBifMaximum}:} 
We have that $|g(b)-g(b')| \leq 2L||b-b'||$ for all $g \in \mathcal{G}_{C,M,L}$ and $b,b' \in [0,1]$. Thus if $g(a)>\epsilon$ for some $a \in [0,1]$ we have that $(g(b))^2 >\epsilon^2/9$ for all $b \in [0,1]: (\min(0,\epsilon-2L|a-b|))^2 \geq \epsilon^2/9$, equivalently $b \in [0,1]: |a-b| \geq \frac{\epsilon}{3L}$. The conclusion that $B_{C,0,L} \geq \frac{\epsilon}{3L}$ then follows immediately. $\square$

The following lemma gives a similar result for the case of $M=1$. In this case the proof relies on the observation that $g'$, the gradient of a function $g \in \mathcal{G}^*_{C,M,L}(a)$, should satisfy $g'(a)=0$, i.e. $a$ should be a maximiser of $g$. The bound on the size of  $B^*_{C,1,L}(a)$ then follows from the Lipschitz property of $g'$. The result holds only for $a$ sufficiently from the edges of $[0,1]$, since $g'(a)$ need not take value 0 to minimise $|\{b:g^2(b)> (\epsilon')^2/9\}|$ if $a$ is close to an edge. Fortunately, however, the impact of these special edge cases is negligible when it comes to bounding the eluder dimension.

\begin{lemma} \label{lem::LBhigherM}
For $a \in [0,1]$ such that $a> \sqrt{\frac{2\epsilon}{3L}}$ and $1-a> \sqrt{\frac{2\epsilon}{3L}}$, and $C,L>0$ we have $B^*_{C,1,L}(a) \geq 2\sqrt{\frac{2\epsilon}{3L}}$. 
\end{lemma}

\noindent \emph{Proof of Lemma \ref{lem::LBhigherM}:} We have that $|g'(b)-g'(b')| \leq 2L||b-b'||$ for all $g \in \mathcal{G}_{C,1,L}$ and $b,b' \in [0,1]$. Thus, for $g$ with $g'(a)=0$, we have $|g'(b)| \leq 2L||a-b||$ for all $b \in [0,1]$. For any $b'<b \in [0,1]$ we have $g(b)-g(b')= \int_{b'}^b g'(x)dx$. It follows that for $0 \leq b<a$ \begin{align*}
g(b)=g(a)-g(a)+g(b) &=g(a)-\int_{b}^a g'(x) dx \\
	&\geq g(a) - \int_b^a 2L(a-x) dx \\
	&= g(a) - La^2 +2Lab - Lb^2 \\
	&> \epsilon' - L(a-b)^2.
\end{align*} A similar argument follows for $a<b\leq 1$ and thus $g(b) > \epsilon'-L||a-b||^2$ for all $b \in [0,1]$ given $g(a)>\epsilon'$ and $g'(a)=0$. It follows that under these conditions we have $g^2(b)>\epsilon^2/9$ for all $b \in [0,1]:(\min(0,\epsilon-L|a-b|^2))^2 \geq \epsilon^2/9$, equivalently $b  \in [0,1]: |a-b| \leq \sqrt{\frac{2\epsilon}{3L}}$. 

If $g'(a)\neq 0$ then $\exists \, c \in [0,1]$ with $g(c)> g(a)>\epsilon'$ and $g'(c)=0$. Then by the logic used for the case with $g'(a)=0$ it follows that $g^2(b) > \epsilon^2/9$ for all $b \in [0,1]: ||b-c|| \leq \sqrt{\frac{1}{L}(g(c)-\epsilon/3)}$. Since $g(c)> \epsilon'$ it follows that if $g(a)>\epsilon'$ then the region such that $g^2(b)>\epsilon^2/9$ is larger if $g'(a)\neq 0$ than if $g(a)=0$. Thus we have $g'(a)=0$ for all $g \in \mathcal{G}^*_{C,1,L}(a)$ and $B_{C,1,L}(a) \geq \sqrt{\frac{2\epsilon}{3L}}$ for all $a \in [0,1]$ such that $a> \sqrt{\frac{2\epsilon}{3L}}$ and $1-a> \sqrt{\frac{2\epsilon}{3L}}$. $\square$

Bounding $B^*_{C,M,L}$ for larger values of $M$ is more involved. To do so we will first define a particular function $\ham \in \mathcal{G}_{C,M,L}$ for each $M \geq 2$ and $a \in [0,1]$ and bound $\Bham $, the size of the region where $\ham$ takes absolute value greater than $\epsilon/3$. We will then show that $\ham$ is in the  set of $B$-minimising functions $\mathcal{G}^*_{C,M,L}$, and thus that $B^*_{C,M,L}(a) = \Bham$. The form of $\ham$ will vary depending on whether $M$ is even or odd. We will first specify $\ham$ for $M$ even.

For $M \geq 2$ even, let $\ham$ be maximised at $a$ with $\ham(a)>\epsilon'$, and let $x_{1,M} = x_{1,a,M} = \max_{x < a, \ham(x)=\epsilon/3} x$ be the point closest to $a$ on the left where $\ham$ takes value $\epsilon/3$. Define $\Delta_M = a-x_{1,M}$, and then further points $y_{1,M} = x_{1,M} - \Delta_M$, $x_{2,M} = a+ \Delta_M$, and $y_{2,M} = a+ 2\Delta_M$. We then specify $\ham$ as a function with $M^{th}$ derivative given as \begin{equation}
\ham^{(M)}(z)=\begin{cases} &2L(x_{1,M}-z), \quad z \in (y_{1,M},a), \\
&2L(z-x_{2,M}), \quad z \in [a,y_{2,M}),
\end{cases} \label{eq::Mderiveven}
\end{equation} and whose lower order derivatives satisfy the following properties: \begin{align}
\ham^{(m)}(x_1)=\ham^{(m)}(x_2)&=0, 2\leq m \leq M, m \text{ even}, \label{eq::lowordercond1} \\
\ham^{(m)}(y_1)=\ham^{(m)}(a)=\ham^{(m)}(y_2)&=0 , m \leq M, m \text{ odd}. \label{eq::lowordercond2}
\end{align} Since $\ham^{(M)}$ is necessarily Lipschitz (by $\ham$'s membership of $\mathcal{G}_{C,M,L}$) this defines the function that can have $\ham^{(M)}(x)=0$ where it crosses $\epsilon/3$ and change most rapidly elsewhere. To bound $B(h_M)$ we first require expressions for the lower order derivatives of $h_M$. Having the restricted behaviour on $\{y_{1,M},x_{1,M},a,x_{2,M},y_{2,M}\}$ means that these functions can be identified from $\ham^{(M)}$ alone. The following lemma specifies the form of these lower order derivatives. We focus on the left of $a$, as a symmetry argument will give an analogous result for the right.

\begin{lemma} \label{lem::lowerderivseven} For the function $\ham$ with $M^{th}$ derivative given by \eqref{eq::Mderiveven}, and whose lower order derivatives satisfy conditions \eqref{eq::lowordercond1} and \eqref{eq::lowordercond2} where $M$ is even, the lower order derivatives are of the form \begin{equation}
\frac{1}{2L}\ham^{(M-m)}(z) = \begin{cases} &j_{m+1}(x_{1,M})-j_{m+1}(z), \quad m \in \{0,2,4,\dots, M\} \\
&j_{m+1}(a)-j_{m+1}(z), \quad m \in \{1,3,\dots,M-1\} \end{cases}
\quad z \in (y_{1,M},a) \label{eq::lowerderivs}
\end{equation}  where \begin{align*}
j_k(z)&=\sum_{i=1}^k \frac{z^i}{i!}(-1)^{k-i}J_{k-i}, \quad k \in \{1,\dots,M+1\}, \\
J_k &= j_k(a\mathbb{I}\{ k \text{ even}\} + x_1\mathbb{I}\{ k \text{ odd}\}),
\end{align*} and $j_0(z)=1$ for all $z \in (y_1,a)$. 
\end{lemma} 

\emph{Proof of Lemma \ref{lem::lowerderivseven}:} We prove this Lemma via an induction argument over $m$. Firstly,  for $m=1$, we have $\frac{1}{2L}h^{(M-m)}(z)=\frac{1}{2L}h^{(M-1)}(z) = \int x_1 - z dz = x_1z - z^2/2 +D$. Since $M-1$ is odd and $h \in \mathcal{G}_{C,M,L}^0(a)$ we have that $h^{(M-1)}(a)=0$ and the integration constant, $D$, must be $a^2/2-x_1a$, i.e. we have 
\begin{displaymath}
\frac{1}{2L}h^{(M-1)}(z) = x_1z - z^2/2 + a^2/2 - ax_1 = j_{2}(a)-j_2(z).
\end{displaymath} Second, for some $m'$ with $2 \leq m' < M$ let us assume that 
\begin{displaymath}
\frac{1}{2L}h^{(M-m')}(z) = J_{m'+1}-j_{m'+1}(z)
\quad z \in (y_1,a).
\end{displaymath}
Finally we consider $h^{(M-m'-1)}$.  We have, \begin{align*}
&\enspace \frac{1}{2L}h^{(M-m'-1)}(z) \\
&= \int J_{m'+1}-j_{m'+1}(z) dz \\
&= \int \sum_{i=1}^{m'+1}\frac{\big(x_1 \mathbb{I}\{m'+1 \text{ odd}\} + a \mathbb{I}\{m'+1 \text{ even}\} \big)^i-z^i}{i!}(-1)^{m'+1-i}J_{m'+1-i}dz \\
&= \sum_{i=1}^{m'+1}\frac{z\big(x_1 \mathbb{I}\{m'+1 \text{ odd}\} + a \mathbb{I}\{m'+1 \text{ even}\} \big)^i}{i!}(-1)^{m'+1-i}J_{m'+1-i} \\
&\quad \quad - \sum_{i=1}^{m'+1} \frac{z^{i+1}}{(i+1)!}(-1)^{m'+1-i}J_{m'+1-i} + D \\
&= \sum_{i=1}^{m'+1}\frac{z\big(x_1 \mathbb{I}\{m'+1 \text{ odd}\} + a \mathbb{I}\{m'+1 \text{ even}\} \big)^i}{i!}(-1)^{m'+1-i}J_{m'+1-i} \\
&\quad \quad - \sum_{i=1}^{m'+1} \frac{z^{i+1}}{(i+1)!}(-1)^{m'+1-i}J_{m'+1-i}  + \sum_{i=1}^{m'+1} \frac{\big(x_1 \mathbb{I}\{m' \text{ odd}\} + a \mathbb{I}\{m' \text{ even}\} \big)^{i+1}}{(i+1)!}(-1)^{m'+1-i}J_{m'+1-i} \\
&\quad \quad - \sum_{i=1}^{m'+1}\frac{(x_1\mathbb{I}\{m' \text{ odd}\} + a\mathbb{I}\{m' \text{ even}\})\big(x_1 \mathbb{I}\{m'+1 \text{ odd}\} + a \mathbb{I}\{m'+1 \text{ even}\} \big)^{i}}{i!}(-1)^{m'+1-i}J_{m'+1-i} \\
&= zJ_{m'+2-1} - \sum_{s=2}^{m'+2} \frac{z^s}{s!}(-1)^{m'+2-s}J_{m'+2-s} - (x_1 \mathbb{I}\{m'+2 \text{ odd}\} + a \mathbb{I}\{m'+2 \text{ even}\})J_{m'+2-1}\\
&\quad \quad + \sum_{s=2}^{m'+2}\frac{\big(x_1 \mathbb{I}\{m'+2 \text{ odd}\} + a \mathbb{I}\{m'+2 \text{ even}\} \big)^{s}}{s!}(-1)^{m'+2-s}J_{m'+2-s} \\
&= \sum_{s=1}^{m'+2} \frac{\big(x_1 \mathbb{I}\{m'+2 \text{ odd}\} + a \mathbb{I}\{m'+2 \text{ even}\} \big)^s-z^s}{s!}(-1)^{m'+2-s}J_{m'+2-s} \\
&= J_{m'+2}- j_{m'+2}(z)
\end{align*} The first equality uses the assumed form of $h^{(M-m')}$, the fourth evaluates the integration constant $D$ based on the knowledge that if $m'+1$ is odd, we will have $h^{(M-m'-1)}(a)=0$ and if $m'+1$ is even, we will have $h^{(M-m'-1)}(x_1)=0$, and the fifth uses a change of variable $s=i+1$.
$\square$

 Since $\ham$ is unimodal, and symmetric about $a$, we have $\Bham > x_{2,M}-x_{1,M}=2(a-x_{1,M})=2\Delta_M$. In the following lemma, we determine the order of $\Bham$ by bounding $\Delta_M$ for each even $M\geq 2$.

\begin{lemma} \label{lem::Mevenhlargeeregion}
For the function $\ham$ with $M^{th}$ derivative given by \eqref{eq::Mderiveven} where $M$ is even, there exist finite constants $K_{1,M},K_{2,M}>0$ such that \begin{displaymath}
K_{1,M} (\epsilon/L)^{1/(M+1)} \leq \Bham \leq  K_{2,M}(\epsilon/L)^{1/(M+1)}.
\end{displaymath}
\end{lemma}

\emph{Proof of Lemma \ref{lem::Mevenhlargeeregion}:} Firstly observe that since $\ham(x_{1,M})=\epsilon/3$ we have by definition that \begin{displaymath}
\ham(a)-\ham(x_{1,M})=\int_{x_{1,M}}^a \ham'(z)dz > \frac{2\epsilon}{3}.
\end{displaymath} Using the definition of $\ham'$ in \eqref{eq::lowerderivs}, we expand the centre term of the above display as follows, \begin{align*}
&\quad \int_{x_{1,M}}^a \ham'(z)dz = \int_{x_{1,M}}^a \ham^{(M-(M-1))}(z)dz \\
&= 2L\int_{x_{1,M}}^a j_M(a)-j_M(z) dz  \\
&= 2L\int_{x_{1,M}}^a j_M(a) - \sum_{i=1}^M \frac{z^i}{i!}(-1)^{M-i}j_{M-i}\big( x_{1,M}\mathbb{I}\{M-i \text{ odd}\} + a\mathbb{I}\{ M-i \text{ even}\}\big) dz \\
&= 2L\Bigg[j_M(a)z - \sum_{i=1}^M \frac{z^{i+1}}{(i+1)!}(-1)^{M-i}j_{M-i}\big( x_{1,M}\mathbb{I}\{M-i \text{ odd}\} + a\mathbb{I}\{ M-i \text{ even}\}\big) \bigg]_{x_{1,M}}^a \\
&= 2L\sum_{i=1}^M \bigg(\frac{a^{i+1}}{i!} - \frac{a^{i+1}}{(i+1)!} - \frac{x_{1,M}a^i}{i!}+ \frac{x_{1,M}^{i+1}}{(i+1)!} \bigg)(-1)^{M-i}j_{M-i}\big( x_{1,M}\mathbb{I}\{M-i \text{ odd}\} + a\mathbb{I}\{ M-i \text{ even}\}\big) \\
&=2L \sum_{i \in \{2,4,\dots,M\}} \bigg(\frac{a^{i+1}}{i!} - \frac{a^{i+1}}{(i+1)!} - \frac{x_{1,M}a^i}{i!}+ \frac{x_{1,M}^{i+1}}{(i+1)!}  \bigg)j_{M-i}\big( a \big) \\
&\quad \quad - 2L\sum_{i \in \{1,3,\dots,M-1\}} \bigg(\frac{a^{i+1}}{i!} - \frac{a^{i+1}}{(i+1)!} - \frac{x_{1,M}a^i}{i!}+ \frac{x_{1,M}^{i+1}}{(i+1)!}  \bigg)j_{M-i}\big( x_{1,M}\big)
\end{align*}

From the definition of the recurrence relation $j$, we have that for $k$ even $j_k(a)$ may be written, for some $\kappa_{l,k}$, $l=1,\dots k$ as $j_k(a)=\sum_{l=1}^k \kappa_{l,k} a^l x_{1,M}^{k-l}$, i.e. for $k$ even $j_k(a)$ is $O(a^k)$ and $O(x_{1,M}^{k-1})$. Similarly for $k$ odd $j_{k}(x_{1,M})$ may be written, for some $\tau_{l,k}$, $l=1,\dots,k$ as $j_k(x_{1,M})=\sum_{l=1}^k \tau_{l,k} x_{1,M}^l a^{k-l}$, i.e. for $k$ odd $j_k(x_{1,M})$ is $O(x_{1,M}^{k})$ and $O(a^{k-1})$.

It follows from this and the above display, that we may write \begin{align*}
\int_{x_{1,M}}^a \ham'(z)dz &=2L \sum_{i \in \{2,4,\dots,M\}} \bigg(\frac{a^{i+1}}{i!} - \frac{a^{i+1}}{(i+1)!} - \frac{x_{1,M}a^i}{i!}+ \frac{x_{1,M}^{i+1}}{(i+1)!}  \bigg) \sum_{l=1}^{M-i} \kappa_{l,M-i}a^{l}x_{1,M}^{M-i-l}  \\
&\quad \quad - 2L\sum_{i \in \{1,3,\dots,M-1\}} \bigg(\frac{a^{i+1}}{i!} - \frac{a^{i+1}}{(i+1)!} - \frac{x_{1,M}a^i}{i!}+ \frac{x_{1,M}^{i+1}}{(i+1)!}  \bigg)   \sum_{l=1}^{M-i} \tau_{l,M-i}x_{1,M}^l a^{M-i-l},
\end{align*} and that there exist constants $H_{M,L,i}$, $i =0, \dots, M+1$ such that \begin{displaymath}
\ham(a)-\ham(x_1)= \sum_{i=0}^{M+1} H_{M,L,i} a^{M+1-i} x_{1,M}^i = O((a-x_{1,M})^{M+1}).
\end{displaymath}

Since $\ham(a)-\ham(x_{1,M})=2\epsilon/(3L)$ we have that $x_{1,M}=a-o((\epsilon/L)^{1/(M+1)})$. By a symmetry argument about $a$ we will also have that $x_{2,M}=a+o((\epsilon/L)^{1/M+1})$. Furthermore, by symmetry of $g'$ about $x_{1,M}$ and $x_{2,M}$ we have that $\ham$ need not fall below $-\epsilon/3$, as $y_{1,M}$ and $y_{2,M}$ may be global minimisers of $\ham$ Thus for $\ham$ as described above, and $M\geq 2$ even, we have \begin{displaymath}
\Bham=2\Delta_M=o((\epsilon/L)^{1/(M+1)})
\end{displaymath} for all $a$ sufficiently far from the edges of $[0,1]$. $\square$

Lemmas \ref{lem::lowerderivseven} and \ref{lem::Mevenhlargeeregion} pertain only to the case where $M$ is even. We must now consider the complementary case of $M$ odd. The function $\ham$ is different, but the argument used to bound $\Bham$ is very similar.

For $M \geq 3$ odd let $\ham$ be a function in $\mathcal{G}^0_{C,M,L}(a)$ with $M^{th}$ derivative specified as 
\begin{equation}
\frac{1}{2L}\ham^{(M)}(z)=\begin{cases} &z-y_{1,M}, \quad z \in (y_{1,M},x_{1,M}), \\
&a-z, \quad z \in [x_{1,M},x_{2,M}), \\
&z-y_{2,M}, \quad z \in [x_{2,M},y_{2,M}),
\end{cases} \label{eq::hModd}
\end{equation} and whose lower order derivatives satisfy conditions \eqref{eq::lowordercond1} and \eqref{eq::lowordercond2}. This is chosen similarly to in the case of $M$ even as the fastest varying function which meets the constraints on the derivatives on $\{y_{1,M},x_{1,M},a,x_{2,M},y_{2,M}\}$. Again, we derive expressions for the lower order derivatives of $\ham$ and focus on the left of $a$, since similar expressions follow for the right hand side by symmetry.

\begin{lemma} \label{lem::derivsModd}
For the function $\ham$ with $M^{th}$ derivative given by \eqref{eq::hModd}, and whose lower order derivatives satisfy conditions \eqref{eq::lowordercond1} and \eqref{eq::lowordercond2} where $M$ is odd, the lower order derivatives are of the form \begin{equation}
\frac{1}{2L}\ham^{(M-m)}(z) = \begin{cases}
&j_{m+1}(z)-J_{m+1}, \quad z \in (y_{1,M},x_{1,M}), \\
&L_{m+1} - l_{m+1}(z), \quad z \in [x_{1,M},a),
\end{cases}
\end{equation} where \begin{align*}
j_k(z) &= \sum_{i=1}^k \frac{z^i}{i!}(-1)^{k-i}J_{k-i}, \quad z \in (y_{1,M},x_{1,M}), \\
J_k &= j_k(y_{1,M}\mathbb{I}\{k \text{ odd}\} + x_{1,M} \mathbb{I}\{ k \text{ even}\}), \\
l_k(z) &= \sum_{i=1}^k \frac{z^i}{i!}(-1)^{k-i}L_{k-i}, \quad z \in [x_{1,M},a) \\
L_k &= l_k(a\mathbb{I}\{k \text{ odd}\} + x_1 \mathbb{I}\{ k \text{ even}\}),
\end{align*} for $k \in \{1,\dots M+1\}$ and where $j_0(z)=l_0(z)=1$ for all $z \in (y_{1,M},a)$.
\end{lemma}

\emph{Proof of Lemma \ref{lem::derivsModd}:} As in the case of $M$ even, we prove this lemma via an induction argument over $m$. Firstly, for $m=1$ we have for $z \in (y_1,x_1)$, $\frac{1}{2L}h^{(M-1)}(z) = \int z-y dz = z^2/2 - yz +D$. Since $M-1$ is even and $h \in \mathcal{G}_{C,M,L}^0(a)$ we have that $h^{(M-1)}(x_1)=0$ and the integration constant, $D$, must be $yx_1-x_1^2/2=-J_2$. For $z \in [x_1,a)$, $\frac{1}{2L}h^{(M-1)}(z) = \int a-z dz = az - z^2/2 + D$, and $D=x_1^2/2 - ax_1=L_2$. Thus, \begin{displaymath}
\frac{1}{2L}h^{(M-1)}(z) = \begin{cases} &j_2(z)-J_2, \quad z \in (y_1,x_1) \\
&L_2 - l_2(z), \quad z \in [x_1,a).
\end{cases}
\end{displaymath} 
Secondly, for some $m'$, $2 \leq m'< M$ we assume that \begin{displaymath}
\frac{1}{2L}h^{(M-m')}(z) = \begin{cases} &j_{m'+1}(z)-J_{m'+1}, \quad z \in (y_1,x_1) \\
&L_{m'+1} - l_{m'+1}(z), \quad z \in [x_1,a).
\end{cases}
\end{displaymath} We now consider $h^{(M-m'-1)}$. For $z \in (y_1,x_1)$ we have, \begin{align*}
&\enspace \frac{1}{2L}h^{(M-m'-1)}(z) \\
&= \int j_{m'+1}(z)-J_{m'+1} dz \\
&= \int \sum_{i=1}^{m'+1}\frac{z^i-\big(y_1 \mathbb{I}\{m'+1 \text{ odd}\} + x_1 \mathbb{I}\{m'+1 \text{ even}\} \big)^i}{i!}(-1)^{m'+1-i}J_{m'+1-i}dz \\
&= \sum_{i=1}^{m'+1} \bigg(\frac{z^{i+1}}{(i+1)!} - \frac{z\big(y_1 \mathbb{I}\{m'+1 \text{ odd}\} + x_1 \mathbb{I}\{m'+1 \text{ even}\} \big)^i}{i!}\bigg)(-1)^{m'+1-i}J_{m'+1-i}  + D \\
&= \sum_{s=2}^{m'+2} \frac{z^s}{s!}(-1)^{m'+2-s}J_{m'+2-s} - zJ_{m'+2-1} + (y_1 \mathbb{I}\{m'+2 \text{ odd}\} + x_1 \mathbb{I}\{m'+2 \text{ even}\})J_{m'+2-1}\\
&\quad \quad - \sum_{s=2}^{m'+2}\frac{\big(y_1 \mathbb{I}\{m'+2 \text{ odd}\} + x_1 \mathbb{I}\{m'+2 \text{ even}\} \big)^{s}}{s!}(-1)^{m'+2-s}J_{m'+2-s} \\
&= j_{m'+2}(z)-J_{m'+2}
\end{align*} This follows the same steps as the proof for $M$ even, but with the opposite sign and slightly different definition of $j$. The proof for $z \in [x_1,a)$ follows the same steps as the above and the proof for $M$ even. The required result follows by induction. $\square$

\begin{lemma} \label{lem::Moddhlargeeregion}
For the function $\ham$ with $M^{th}$ derivative given by \eqref{eq::hModd} where $M$ is odd, there exist finite constants $K_{3,M}, K_{4,M}>0 $ such that  \begin{displaymath}
K_{3,M}(\epsilon/L)^{1/(M+1)} \leq \Bham \leq K_{4,M}(\epsilon/L)^{1/(M+1)}
\end{displaymath}
\end{lemma}

\emph{Proof of Lemma \ref{lem::Moddhlargeeregion}:} By the definition of $x_{1,M}$ we have $\ham(a)-\ham(x_{1,M}) = \int_{x_{1,M}}^a \ham'(z)dz > 2\epsilon/3$. We rewrite the LHS of this relation as follows, \begin{align*}
\int_{x_{1,M}}^a \ham'(z)dz &= 2L\int_{x_{1,M}}^a L_{M} - l_M(z)dz \\
&= 2L \bigg[L_M z - \sum_{i=1}^M \frac{z^{i+1}}{(i+1)!}(-1)^{M-i}L_{m-i} \bigg]_{z=x_{1,M}}^a \\
&= 2L\sum_{i=1}^M \bigg( \frac{a^{i+1}}{i!}- \frac{a^{i+1}}{(i+1)!} -\frac{x_{1,M}a^i}{i!} + \frac{x_{1,M}^{i+1}}{(i+1)!}\bigg) (-1)^{M-i} L_{M-i}.
\end{align*} This is the same expression derived for $h_{a,M}(a) -h_{a,M}(x_{1,M})$ as in the $M$ even case, and thus the same conclusion follows. $\square$

The combined insight from Lemmas \ref{lem::Mevenhlargeeregion} and \ref{lem::Moddhlargeeregion} is that for any $M \geq 2$ and $a \in [2\Delta_M,1-2\Delta_M]$ there exists a function $\ham \in \mathcal{G}_{C,M,L}$ with $\Bham = o((\epsilon/L)^{1/(M+1)})$. We will demonstrate that this $o((\epsilon/L)^{1/(M+1)})$ result is optimal, in the sense that $B^*_{C,M,L}(a)=o((\epsilon/L)^{1/(M+1)})$ also.

Firstly, notice that $g'(a)=0$ necessarily for all $g \in \mathcal{G}^*_{C,M,L}(a)$. If for some $g \in \mathcal{G}_{C,M,L}$ with $g(a)> \epsilon'$, $g'(a) \neq 0$ then either there exists $c \in [0,1]$ such that $g(c)> g(a)$ and $g'(c)=0$ or else $g(b)>g(a)$ for all $b$ in either $[0,a)$ or $(a,1]$. If the first event happens, by the same theory that says $\Delta_M$ is increasing in $g(a)$, there will be a region of width greater than $2\Delta_M$ centred $c$ where $g(b)>\epsilon/3$. If the second event happens, $B(g)$ is plainly greater than $2\Delta_M$ since $a > 2\Delta_M$ and $1-a > 2\Delta_M$. We therefore deduce that $g'(a)=0$ for all $g \in \mathcal{G}^*_{C,M,L}(a)$ since $\Bham < B(g)$ for any $g$ with $g(a)>\epsilon'$ and $g'(a) \neq 0$.

Next we observe that $\Bham$ is the optimal value of $B(g)$ among functions $g \in \mathcal{G}_{C,M,L}$ with $g(a)>\epsilon'$ and derivatives constrained as in \eqref{eq::lowordercond1} and \eqref{eq::lowordercond2}. For any such $g \in \mathcal{G}_{C,M,L}$ it is true that $B(g) = x_{2,g}-x_{1,g}$ where $x_{1,g} = \max_{x < a: g(x)=\epsilon/3} x$ and similarly $x_{2,g} = \min_{x > a: g(x)=\epsilon/3} x$. For $h_{a,M}$, we know that $x_{1,h_{a,M}}= a - \Delta_M$ and $x_{2,h_{a,M}}=a + \Delta_M$, thus that $x_{2,h_{a,M}}-x_{1,h_{a,M}} = 2\Delta_M$. The value of $\Delta_M$ is determined by $\ham'$, which we have previously pointed out changes at the fastest rate possible for a function with derivatives constrained according to \eqref{eq::lowordercond1} and \eqref{eq::lowordercond2}. Thus for any other function $g$ with derivatives constrained according to \eqref{eq::lowordercond1} and \eqref{eq::lowordercond2}, $x_{2,g}-x_{1,g} \geq 2\Delta_M$ and $B(g) \geq \Bham$. 

On the other hand, functions whose derivatives are not constrained according to \eqref{eq::lowordercond1} and \eqref{eq::lowordercond2} may have $x_{2,g}-x_{1,g}<2\Delta_M$. However, such functions will take value less than $-\epsilon/3$ at some points in $[0,1]$. That is to say $B(g) \neq x_{2,g}-x_{1,g}$ for such functions, since $y_{1,g}$ and $y_{2,g}$ cannot not be global minimisers. We will show that $B(g)> \Bham$ for functions $g \in \mathcal{G}_{C,M,L}$ with $g(a)> \epsilon$ and $x_{2,g}-x_{1,g}>2\Delta_M$.

As before, we will consider the left hand side of $a$ and allow the behaviour on the right hand to be explained by a symmetry argument. If, for a function $g \in \mathcal{G}_{C,M,L}$ with $g(a)>\epsilon'$ and $g'(a)=0$ (otherwise it would not be optimal anyway) we have $x_{1,g} > x_{1,M}$ - i.e. the point on the left where $g$ takes value $\epsilon/3$ is nearer to $a$ than under $\ham$ - then we have that $\int_{x_{1,g}}^a g'(z)dz > \int_{x_{1,g}}^a h'_M(z)dz$. Since $g'(a)=\ham'(a)=0$, this implies that $g''(z)<\ham''(z)$ over $[x_{1,g},a]$ and that  $g'(y_{1,g})=0$ is not possible. There instead exists a point $y_{1,min}<y_{1,g}$ with $g(y_{1,min})< -\epsilon/3$ and $g'(y_{1,min})=0$. The contribution to $B(g)$ from the left side of $a$ is then at least $a-x_{1,g} + 2(y_{1,g}-y_{1,min})$. $y_{1,g}-y_{1,min} =x_{1,g}-x_{1,M}$ by the smoothness properties of functions in $\mathcal{G}_{C,M,L}$ and thus the contribution to $B(g)$ from the left of $a$ will be greater than that of $\Bham$. A similar result follows on the right of $a$, and we thus have that $B(g)> \Bham$ for functions with $x_{2,g}-x_{1,g}< 2\Delta_M$. If $x_{2,g}-x_{1,g} > 2\Delta_M$ then the function $g$ is obviously not optimal.

By showing that $\ham$ is optimal amongst functions with similarly constrained derivatives, and that $\Bham \leq B(g)$ for functions $g$ without these constraints, we have therefore demonstrated that $B^*_{C,M,L}(a) = o((\epsilon/L)^{1/(M+1)})$ for $a \in [2\Delta_M,1-2\Delta_M]$.

We complete the proof of Theorem \ref{prop::Lipeluder} by noticing that if $k = 9/B_{C,M,L}^*+2$ then for any sequence $a_{1:k} \in [0,1]$ there must exist an index $j \in \{1,\dots,k\}$ such that $a_j \in [2\Delta_M,1-2\Delta_M]$ and there exist distinct at least 9 distinct points $a_{l_i}$, $l_i \in \{1,\dots,j-1\}$, $i=1,\dots,9$ with $|a_j-a_{l_i}| \leq B^*_{C,M,L}/2$. Then if $g(a_j)> \epsilon'$ and $g \in \mathcal{G}_{C,M,L}$ it follows that $(g(a_{l_i}))^2 > (\epsilon')^2/9$ for $i \in \{1,\dots 9\}$ and $\sum_{i=1}^{j-1} (g(a_i))^2 > (\epsilon')^2$.

Therefore if $k \geq 9/B^*_{C,M,L}+2$ there exists no sequence $a_{1:k} \in [0,1]^k$ such that $w_\tau(a_{1:\tau},\epsilon')>\epsilon'$ for every $\tau \leq k$, and thus $dim_{E}(\mathcal{F}_{C,M,L},\epsilon) \leq k =o((\epsilon/L)^{1/(M+1)})$. $\square$

\newpage
\section{Proof of the Regret Lower Bound}
In this section we provide a proof of the lower bound on regret for CABs whose reward functions have $M>0$ Lipschitz derivatives, restated below.

\textbf{Theorem 5} \emph{Let \texttt{ALG} be any algorithm for the CAB problem with reward function in $\mathcal{F}_{C,M,L}$. There exists a problem instance $\mathcal{I}=\mathcal{I}(x^*,\delta)$ for some $x^* \in [0,1]$ and $\delta>0$ such that} \begin{displaymath}
\mathbb{E}(R(T)|\mathcal{I}) \geq \Omega(T^{(M+2)/(2M+3)}).
\end{displaymath} 

We first state a lower bound on regret for stochastic $K$-armed bandits, on which the proof of Theorem \ref{thm::ext_lower_bound} relies. This result, presented below, is a generalisation of the well-known $\Omega(\sqrt{KT})$ problem independent regret lower bound in Theorem 5.1 of \cite{AuerEtAl2002non}, and its proof can be extracted from the proof of the original result. The version we state is from \cite{Slivkins2019}, but a very similar generalisation of Auer et al's theorem was originally presented in \cite{BubeckEtAl2011}.

\textbf{Theorem 6} (Theorem 4.3 of \cite{Slivkins2019}) \emph{Consider stochastic bandits with $K$ arms and horizon $T$. Let \texttt{ALG} be any algorithm for this problem. Pick any positive $\delta \leq \sqrt{c_0K/T}$, where $c_0$ is a small universal constant. Then there exists a problem instance $\mathcal{J}=\mathcal{J}(a^*,\delta)$, $a^* \in [K]$, such that} \begin{displaymath}
\mathbb{E}(R(T)|\mathcal{J}) \geq \Omega(\delta T).
\end{displaymath}

By relating the regret of algorithms for the CAB problems of interest to that of algorithms for particular MAB problems, we will be able to utilise Theorem 6 to prove Theorem \ref{thm::ext_lower_bound}. 

\emph{Proof of Theorem \ref{thm::ext_lower_bound}:} 
We define the CAB problem instance $\mathcal{I}(x^*,\delta,M)$ as that with reward function $\nu_{x^*,\delta,M} \in \mathcal{F}_{1,M,L}$, whose form we shall specify below. The function $\nu_{x^*,\delta,0}$ is identical to the function $\mu$ used in the original lower bound proof for Lipschitz bandits, and stated as equation \eqref{eq::worstcaserewardfunc} in the main text. For clarity we define, \begin{displaymath}
\nu_{x^*,\delta,0}(x)=  \begin{cases}\begin{alignedat}{2} 
&0.5 + \delta -L|x^*-x| \quad \quad &x: |x^*-x| \leq \delta/L, \\
&0.5 \quad &\text{ otherwise.} \end{alignedat}\end{cases}
\end{displaymath}

For general $M \geq 1$, $\nu_{x^*,\delta,M}:[0,1] \rightarrow [0,5,0.5+\delta]$ are symmetric (around $x^*$), unimodal, \emph{bump functions} with  $\nu_{x^*,\delta,M}(x^*)=0.5+\delta$, and whose $(M+1)^{th}$ derivatives, $\nu_{x^*,\delta,M}^{(M+1)}$, are piecewise-constant functions from $[0,1]$ to $\{-L,0,L\}$. In particular, they are the functions which minimise the width of such a \emph{bump}, the region where the function takes value greater than 0.5. Let $\mathcal{F}_{C,M,L}^{[a,b]}$ be the restriction of $\mathcal{F}_{C,M,L}$ to its elements which are defined $[0,1] \rightarrow [a,b]$ for $0\leq a \leq b \leq C$. 
The functions of interest may then be defined as follows:\begin{equation}
    \nu_{x^*,\delta,M} \in \argmin_{\nu \in \mathcal{F}_{C,M,L}^{[0.5,0.5+\delta]}: \nu(x^*)=0.5+\delta} \int_0^1 |0.5-\nu(x)|dx. \label{eq::nudef}
\end{equation} 

We do not require the exact form of the functions $\nu_{x^*,\delta,M}$ for the analysis that follows, and as they are complex to write in closed form we will not do so. Their key property, however, is given in the following lemma. 
\begin{lemma}
For any $M \in \mathbb{N}$, function $\nu_{x^*,\delta,M}$ as defined in \eqref{eq::nudef} there exists a finite constant $c_{1,M}>0$ such that \begin{displaymath}
\nu_{x^*,\delta,M}(x) \begin{cases}\begin{alignedat}{2} 
&=0.5  \quad \quad &x: |x^*-x| > c_{1,M}(\delta/L)^{1/(M+1)}, \\
&>0.5 \quad &\text{ otherwise.} \end{alignedat}\end{cases}
\end{displaymath}
\end{lemma}  

\emph{Proof.} Two properties are apparent from the definition of $\nu_{x^*,\delta,M}$. Firstly that the $(M+1)^{th}$ derivative of $\nu_{x^*,\delta,M}$ is piecewise-constant on $\{-L,0,L\}$, since otherwise the rate of change of lower order derivatives could be more rapid, and the width of the bump could be smaller. Secondly, by the fundamental theorem of calculus, we have that the first derivative satisfies $\int_0^{x^*} \nu_{x^*,\delta,M}'(x)dx=\delta$. However, since the function $\nu_{x^*,\delta,M}$ is constant on a large proportion of the unit interval, we also have $\int_{y}^{x^*} \nu_{x^*,\delta,M}'(x)dx = \int_0^{x^*} \nu_{x^*,\delta,M}'(x)dx=\delta$ for all $y \in [0,x_{max}]$ for some $x_{max} < x^*$. The width of the bump is $2(x^*-x_{max})$.

The Cauchy formula for repeated integration tells us that we may write the first derivative in terms of an antiderivative of a higher order derivative, specifically, to relate the first and $(M+1)^{th}$ derivatives, we have \begin{displaymath}
\nu_{x^*,\delta,M}'(x) = \frac{1}{(M-1)!}\int_0^x (x-t)^{M-1}\nu_{x^*,\delta,M}^{(M+1)}(t)dt.
\end{displaymath} As the $(M+1)^{th}$ derivative is piecewise constant, it follows that $\nu_{x^*,\delta,M}'$ is an $O(x^M)$ piecewise polynomial, identifiable given $\nu_{x^*,\delta,M}^{(M+1)}$ by the property that $\nu_{x^*,\delta,M}'(x^*)=0$ (which follows from the unimodality of $\nu_{x^*,\delta,M}$). Similarly, $\nu_{x^*,\delta,M}$ must be a $O(x^{M+1})$ piecewise polynomial and $x_{max}$ may be written as being $x^*-O((\delta/L)^{1/(M+1)})$, completing the proof. $\square$

The idea of the proof of Theorem \ref{thm::ext_lower_bound} is to derive a reward distribution such that the expected reward is given by $\nu$ but that the regret of any algorithm applied to the problem with that reward distribution is bounded below by that incurred when playing a related $K$-armed bandit problem. This is the same approach used to prove Theorem \ref{thm::standard_lip_lower_bound}, but here the proof is adapted to handle the more complex reward functions.

Fix $K \in \mathbb{N}$ to be defined later, and let $\delta=L(1/2c_{1,M}K)^{M+1}$. We introduce a function $f_{\delta}: [K] \rightarrow [0,1]$ which will be used to associate arms of a particular $K$-armed bandit problem with points in the CAB action space. We define this function as follows, \begin{equation}
f_{\delta}(a) :=(2a-1)\delta
\end{equation} Now let $\mathcal{J}(a^*,\delta,M)$ be the $K$-armed bandit problem instance where for $a \in [K]$ we have $\mu_a = \nu_{x^*,\delta,M}(f_\delta(a))$. By the definition of $f_\delta$ we we have that $\mu_{a^*}=0.5+\delta$ and that $\mu_a=0.5$ for $a \in [K]$, $a\neq a^*$.

Let \texttt{ALG} be any algorithm for the CAB problem instance $\mathcal{I}(x^*,\delta,M)$ - i.e. a rule which selects actions $x_1,x_2,\dots \in [0,1]$. Then define \texttt{ALG'} as an associated algorithm which for the MAB problem instance $\mathcal{J}(a^*,\delta)$ which makes decisions on the basis of those of \texttt{ALG} as follows. When \texttt{ALG} selects an action $x_t \in [0,1]$, \texttt{ALG'} selects an action $a_t=a(x_t) \in [K]$ such that $x_t \in (f_{x^*,\delta,M}(a_t)-1/2K,f_{x^*,\delta,M}(a_t)+1/2K]$. By the definition of the MAB problem, \texttt{ALG'} receives reward $r$ which is a Bernoulli random variable with parameter $\mu_{a_t}$. \texttt{ALG} receives reward $r_x$ defined as follows, \begin{equation}
r_x = \begin{cases} &r \quad \text{ with probability } p_x \in [0,1], \\
&X \quad \text{ otherwise,} \end{cases}
\end{equation} where $X$ is a Bernoulli variable with parameter $0.5$.

Choosing the probability $p_x$ as follows, \begin{displaymath}
p_x = \frac{0.5-\nu_{x^*,\delta,M}(x)}{0.5-\mu_{a(x)}}
\end{displaymath} we then have \begin{align*}
\mathbb{E}(r_x| x) &=(1-p_x)\mathbb{E}(X) + p_x \nu_M(f(a(x))) \\
&= 0.5 -0.5p_x + p_x\nu_{x^*,\delta,M}(f(a(x)))         \\
&= \nu_{x^*,\delta,M}(x) 
\end{align*} 

The construction of \texttt{ALG} and \texttt{ALG'} ensures that \begin{displaymath}
\nu_{x^*,\delta,M}(x_t) = \mathbb{E}(r_{x_t}|x_t) \leq \mathbb{E}(r|a_t) = \mu_{a_t}.
\end{displaymath} It follows that $\sum_{t=1}^T \nu_{x^*,\delta,M}(x_t) \leq \sum_{t=1}^T \mu_{a_t}$ and since $\nu_{x^*,\delta,M}(x^*)=\mu_{a^*}$ we have \begin{displaymath}
\mathbb{E}(R(T)|\mathcal{I}) \geq \mathbb{E}(R'(T)|\mathcal{J}).
\end{displaymath} Thus any lower bound on the regret of \texttt{ALG'} on $\mathcal{J}$ serves as a lower bound on the regret of \texttt{ALG} on $\mathcal{I}$. Recall, that Theorem 6 can be used to lower bound the regret of any algorithm for $\mathcal{J}$, and thus all that remains is to specify a choice of $\delta$ to achieve the required bound.

Theorem 6 requires $\delta \leq \sqrt{c_0K/T}$, so we select \begin{displaymath}
K=\Bigg(\frac{T}{c_0}\bigg(\frac{1}{(2c_{1,M})^{2M+2}}\bigg)\Bigg)^{{1}/{(2M+3)}}
\end{displaymath} so that this is satisfied. Then by Theorem 6, there exists an instance $\mathcal{J}$ such that \begin{displaymath}
\mathbb{E}(R'(T)|\mathcal{J}) \geq \Omega(\delta T) = \Omega\bigg(T^{1-\frac{M+1}{2M+3}}\bigg)
\end{displaymath} and therefore $\mathbb{E}(R(T)|\mathcal{I}) \geq \Omega(T^{(M+2)/(2M+3)})$ as required. $\square$

\newpage
\section{Finite and (Generalised) Linear Function Classes} \label{sec::finitegenlin}
Equipped with the general bound of Theorem \ref{prop::generaleluderregret}, providing regret bounds for specific function classes and action sets is a matter of bounding the eluder dimension $dim_E(\mathcal{F},\kappa(T))$ and ball width function $\beta^*_t(\mathcal{F},\delta,\alpha,\lambda)$. 
In the setting of sub-Gaussian reward noise, \cite{RussoVanRoy2014} provide bounds for $dim_E(\mathcal{F},T^{-1})$ and the sub-Gaussian version of the ball-width function for three simple function settings: finitely many actions, linear function classes, and generalised linear function classes. We present analogous results for these settings under sub-exponential reward noise.

\subsection{Eluder Dimension}
The eluder dimension does not depend on the reward noise, and thus translates directly from the work of \cite{RussoVanRoy2014}. Thus for finite function classes, we may bound the eluder dimension as $dim_E(\mathcal{F},\epsilon) \leq |\mathcal{A}|$ for all $\epsilon>0$. For linear reward functions $f_0(a) = \theta^T\phi(a)$ where $\theta \in \Theta \subset \mathbb{R}^d$ such that $\mathcal{F}=\{f_\rho, \rho \in \Theta\}$. If there exist constants $S$ and $\gamma$, such that $||\rho||_2 \leq S$ and $||\phi(a)||_2\leq \gamma$ for all $a \in \mathcal{A}$ then the eluder dimension may be bounded as $dim_E(\mathcal{F},\epsilon) \leq 3d\frac{e}{e-1}\log(3+3(\frac{2S}{\epsilon})^2)+1$. Finally, consider generalised linear reward functions $f_0(a)=g(\theta^T\phi(a))$ where again $\theta \in \Theta \subset \mathbb{R}^d$ and $\mathcal{F}=\{f_\rho,\rho \in \Theta\}$, and where $g(\cdot)$ is a differentiable and strictly increasing function. If there exist constants $\underline{h}, \overline{h}, S$ and $\gamma$ such that for all $\rho \in \Theta$ and $a \in \mathcal{A}$, $0 \leq \underline{h} \leq g'(\rho^T\phi(a)) \leq \overline{h}$, $||\rho||_2 \leq S$, and $||\phi(a)||_2 \leq \gamma$  then the eluder dimension can be bounded as $\dim_E(\mathcal{F},\epsilon)\leq 3dr^2\frac{e}{e-1}\log(3r^2+3r^2(\frac{2S\overline{h}}{\epsilon})^2)+1$, where $r=\sup_{\tilde\theta,a}g'(<\phi(a),\tilde\theta>)/\inf_{\tilde\theta,a}g'(<\phi(a),\tilde{\theta}>)$ bounds the ratio between the maximal and minimal slope of $g$.

\subsection{Ball Width Function}
For finite function classes, and $\alpha=0$ we have $\beta^*_n(\mathcal{F},\delta,0,\lambda)=\frac{\log(|\mathcal{F}|/\delta)}{\lambda(1-2\lambda\sigma^2)}$. For both the class of linear and generalised linear reward functions we have $\log N(\alpha,\mathcal{F},||\cdot||_\infty)=O(d\log(1/\alpha))$ from \cite{RussoVanRoy2014}. It follows from the definition \eqref{eq::bwf} that in both cases $\beta^*_T(\mathcal{F},\delta,1/T^2,\lambda)=O(d\log(T/\delta))$.

\subsection{Regret Bounds}
As a result, for finite function classes we have,  \begin{equation*}
BR(T) \leq 1+ (|\mathcal{A}|+1)C + 4\sqrt{\frac{|\mathcal{A}|\log(2|\mathcal{F}|T)}{\lambda(1-2\lambda\sigma^2)}T}.
\end{equation*} For linear and generalised linear function classes we have, for $\delta \leq 1/2T$, \begin{equation*}
BR(T) = O\Big(d\log(T)+\sqrt{d^2\log(T+T/\delta)T}\enspace\Big).
\end{equation*}

The orders, with respect to $T$, of these bounds match those of Russo and Van Roy's bounds for the sub-Gaussian case, and are optimal up to the small contribution of the logarithmic factors.

\end{document}